\documentclass[11pt]{article}
\pdfoutput=1

\usepackage{mathrsfs}
\usepackage{amsmath,amssymb}
\usepackage{bm}
\usepackage{natbib}
\usepackage[usenames]{color}
\usepackage{amsthm}
\usepackage{bbm}
\usepackage{dsfont}

\usepackage{multirow} 
\usepackage{enumitem}
\usepackage{subfig}
\usepackage{graphicx}

\usepackage[colorlinks,
linkcolor=red,
anchorcolor=blue,
citecolor=blue
]{hyperref}

\usepackage{mylatexstyle}

\usepackage{setspace}
\usepackage[left=1in, right=1in, top=1in, bottom=1in]{geometry}

\usepackage{xcolor}

\ifdefined\final
\usepackage[disable]{todonotes}
\else
\usepackage[textsize=tiny]{todonotes}
\fi
\title{\huge Hidden Layer Distribution}

\author
{
Chenyang Zhang\thanks{\scriptsize Department of Statistics and Actuarial Science, School of Computing and Data Science, The University of Hong Kong; {\tt chyzhang@connect.hku.hk}}
\qquad 
Peifeng Gao\thanks{\scriptsize Department of Computer Science, School of Computing and Data Science, The University of Hong Kong; {\tt gaopeifeng@connect.hku.hk}}
\qquad 
Difan Zou\thanks{\scriptsize Department of Computer  Science, School of Computing and Data Science, \& Institute of Data Science, The University of Hong Kong;
 {\tt dzou@cs.hku.hk}}
\qquad
Yuan Cao\thanks{\scriptsize Department of Statistics and Actuarial Science, School of Computing and Data Science,
	The University of Hong Kong;  {\tt yuancao@hku.hk}}
}

\date{}

\begin{document}

\title{\huge Gradient Descent Robustly Learns the Intrinsic Dimension of Data in Training Convolutional Neural Networks}


\maketitle

\begin{abstract}
Modern neural networks are usually highly over-parameterized. Behind the wide usage of over-parameterized networks is the belief that, if the data are simple, then the trained network will be automatically equivalent to a simple predictor. Following this intuition, many existing works have studied different notions of ``ranks'' of neural networks and their relation to the rank of data. In this work, we study the rank of convolutional neural networks (CNNs) trained by gradient descent, with a specific focus on the robustness of the rank to image background noises. Specifically, we point out that, when adding background noises to images, the rank of the CNN trained with gradient descent is affected far less compared with the rank of the data. We support our claim with a theoretical case study, where we consider a particular data model to characterize low-rank clean images with added background noises. We prove that CNNs trained by gradient descent can learn the intrinsic dimension of clean images, despite the presence of relatively large background noises. We also conduct experiments on synthetic and real datasets to further validate our claim.
\end{abstract}

\section{Introduction}

Neural networks have become a cornerstone in modern machine learning, demonstrating remarkable performance across various domains. A common characteristic of modern networks is their tendency to be highly over-parameterized. Interestingly, it has been demonstrated that over-parameterized models trained by standard optimization algorithms  exhibit a preference for simplicity \citep{DBLP:journals/jmlr/SoudryHNGS18, pmlr-v99-ji19a,DBLP:conf/aistats/NacsonSS19,pmlr-v80-gunasekar18a,gunasekar2017implicit,li2018algorithmic,arora2019implicit,razin2020implicit,chizat2020implicit, phuong2020inductive}: if the training data can be fitted well by a simple predictor, then after training, an over-parameterized model may effectively reduce to this simple predictor.


A notable line of recent works has considered notions of ``ranks'' to characterize how simple the over-parameterized neural network after training is \citet{gunasekar2017implicit,li2018algorithmic,arora2019implicit,chizat2020implicit,zhou2022towards,jacot2022implicit}. Specifically, \citet{gunasekar2017implicit,li2018algorithmic,arora2019implicit} showed that when training over-parameterized matrix factorization models and linear neural networks, gradient descent has an implicit bias towards a low-rank solution. Similar conclusions have also been demonstrated for more nonlinear networks. Specifically,  \citet{chizat2020implicit} showed that the effective hidden layer neurons in a two-layer neural network is sparse. \citet{zhou2022towards} empirically demonstrated that the hidden neural weight vectors condense on isolated orientations when learning easy tasks, and provided explanations of this phenomenon with theoretical case studies. \citet{jacot2022implicit} further formulated two notions of ranks, namely the Jacobian rank and the Bottleneck rank, for vector-valued neural networks, and demonstrated that over-parameterized networks tend to achieve small ranks. 




\begin{figure}[ht!]
	\begin{center}
 		\hspace{-0.2in}
            \subfloat[ MNIST]{\includegraphics[width=0.4\textwidth]{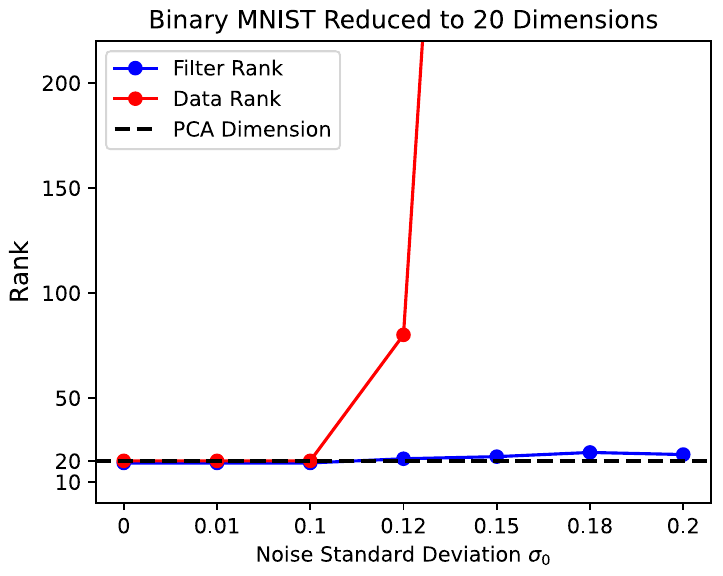}}
            \subfloat[CIFAR-10 ]{\includegraphics[width=0.4\textwidth]{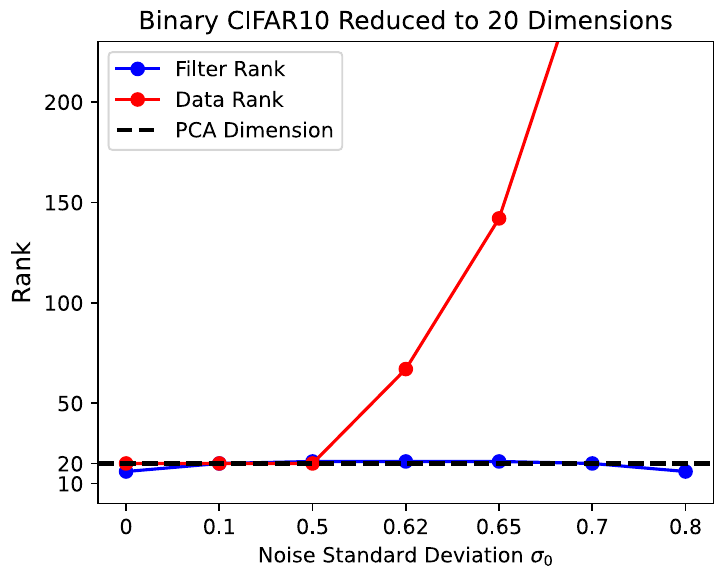}}
	\end{center}
	\vskip -12pt
        \caption{Ranks of data and filters under different noise levels. In (a), we perform a principal component analysis (PCA) to a subset of MNIST images to reduce the number of principal components to 20, which represents the rank of clean data. We then add background noise patches around the obtained low-rank image, and train a two-layer CNN with fixed second layer weights until convergence. We then calculate both the rank of the noisy images and the rank of the matrix consisting of all the convolutional filters of the CNN. When calculating ranks, eigenvalues smaller than $1/100$ of the largest eigenvalue are ignored. The curves of filter rank and data rank with respect to the noise level are plotted. In (b), we conduct a similar set of experiments on the CIFAR-10 data set.}
	\label{fig:intro}
\end{figure}

 In this work, we aim to study the ``ranks'' of two-layer convolutional neural networks (CNN) when learning from data sets with a clean low-rank structure from a new perspective: we examine the robustness of the neural network rank when background noises of increasing levels are added to the original low-rank data. Interestingly, we can draw the following conclusion:
 \begin{center}\textit{The rank of CNN is more robust to background noises, compared with the rank of data.}
 \end{center}

An illustration of this claim is given in Figure~\ref{fig:intro}. From Figure~\ref{fig:intro}, it is evident that for both MNIST and CIFAR-10 images, as the level of background noise increases, the rank of the matrix composed of the data inputs rises sharply, eventually leading to an explosion. In comparison, the rank of the matrix representing CNN filters consistently remains close to the PCA rank (i.e. rank of clean data), demonstrating much less fluctuation. Therefore, these empirical observations match the claim that the rank of the CNN filters is more robust to background noise than the rank of data.

The major contributions of this paper are as follows:
\begin{itemize}[leftmargin = *]
    \item We reveal the ``rank robustness'' phenomenon in training convolutional neural networks. In particular, we add different levels of background noises to clean low-rank data and then use CNNs to fit these noisy data. We observe that, even if a significant amount of noise has been added which causes the stable rank of the data to explode, the stable rank of the CNN filters can still remain around the rank of the clean data. This suggests that the rank of CNN is more robust to the noise compared to the rank of data.
    
    \item We theoretically prove that the observed phenomenon happens when training a two-layer CNN on a data model with multiple patches, where some of the patches contain the object content for classification while others are filled with background noise. More specifically, we show that under a wide range of noise levels, the CNN model will learn the rank of the clean data.
    In comparison, we also show that under the same noise levels, the data rank can provably explode.
    \item Our theoretical analysis is based on a careful examination of the CNN training dynamics. Specifically, we develop theoretical tools to demonstrate that different convolutional filters of the CNN are updated by gradient descent in distinct directions, with each filter's direction determined by its random initialization. We believe that these novel analysis tools can be applied to the study of CNNs from other perspectives as well and are of independent interest.
    
    
\end{itemize}

\noindent\textbf{Notation.} Given two sequences $\{x_n\}$ and $\{y_n\}$, we denote $x_n = O(y_n)$ if there exist some absolute constant $C_1 > 0$ and $N > 0$ such that $|x_n|\le C_1 |y_n|$ for all $n \geq N$. Similarly, we denote $x_n = \Omega(y_n)$ if there exist $C_2 >0$ and $N > 0$ such that $|x_n|\ge C_2 |y_n|$ for all $n > N$. We say $x_n = \Theta(y_n)$ if $x_n = O(y_n)$ and $x_n = \Omega(y_n)$ both holds. We use $\tilde O(\cdot)$, $\tilde \Omega(\cdot)$, and $\tilde \Theta(\cdot)$ to hide logarithmic factors in these notations respectively. Moreover, we denote $x_n=\poly(y_n)$ if $x_n=O( y_n^{D})$ for some positive constant $D$, and $x_n = \polylog(y_n)$ if $x_n= \poly( \log (y_n))$. For two scalars $a$ and $b$, we denote $a \vee b = \max\{a, b\}$ and $a\wedge b =\min\{a,b\}$. For any $n\in \NN_+$, we use $[n]$ to denote the set $\{1, 2, \cdots, n\}$. We denote $\diag(a_1, a_2, \cdots, a_n)$ the diagonal matrix with $a_i$'s being its diagonal entries.

\section{Problem setup}\label{sec:ProblemSetup}

As outlined in the introduction, this paper aims to provide a theoretical explanation for an intriguing empirical phenomenon concerning the ``rank" of the filters in a convolutional neural network (CNN) when applied to noisy data. However, the classic definition of rank is not suitable for our study, as it is highly susceptible to minor perturbations: even the slightest noise can cause the data to achieve full rank. While in experiments, the ``approximate rank” of a matrix is usually clear by looking at the distribution of singular values, a theoretical characterization of this phenomenon requires a mathematically more rigorous definition. Fortunately, the definition of ``stable rank'' meets this requirement, which could be regarded as a ``continuous counterpart'' of the classic rank \citep{georgiev2021impact}. Stable rank was first introduced in \citet{rudelson2007sampling} for studying the low-rank approximation of large matrices, and we provide a formal definition in the following.


\begin{definition}\label{def:stable_rank}
For a matrix $\Ab \in \RR^{d_1\times d_2}$, the stable rank  of $\Ab$ is defined as 
\begin{align*}
    \sr(\Ab) := \frac{\|\Ab\|_F^2}{\|\Ab\|_{2}^2}
\end{align*}
\end{definition}
Compared to the classic definition of rank, stable rank is analogous in many aspects but considerably more well-behaved: it is more stable to small perturbations and it is differentiable. 
To illustrate its robustness to small perturbations, consider a diagonal matrix 
$\Ab \in \RR^{d\times d}$ with its first diagonal with its first diagonal entry set to $1$ and all other entries equal to $\epsilon$. As $\epsilon \to 0$, the classic rank of $\Ab$ remains $d$ until $\epsilon$ exactly equals to 0. In contrast, the stable rank of $\Ab$ equals $1+(d-1)\epsilon$, approaching $1$ smoothly as $\epsilon$ decreases \citep{frei2022implicit}. Therefore, stable rank is a suitable quantity to characterize the impact of noise at different levels. 


Besides the definition of the stable rank, we also need to give a concrete mathematical description of the image data with noisy backgrounds. In order to enable the analysis of CNNs and demonstrate out empirical observations in Figure~\ref{fig:intro}, we particularly hope to 
mathematically capture the following nature of the data: 
\begin{enumerate}[nosep]
    \item[(i)] The data consists of multiple ``patches'' (blocks of ``pixels'' with localized features).
    \item[(ii)] Some patches are ``clean object patches'', the others are ``noisy background patches''.
    \item[(iii)] The ``clean object patches'' share a low-rank structure.
\end{enumerate}
Motivated by this, we consider performing a case study based on a particular data model to theoretically demonstrate the robustness of the stable rank of the CNNs. The definition of the data model $\cD$ is given as follows:

\begin{definition}[Low-rank clean images with noisy backgrounds]\label{def:data}
Let $\ \cU_{+1}, \cU_{-1} \subset \RR^d$ be two disjoint sets of orthonormal vectors.
A data pair $(\Xb,y)$ is generated as follows. The label $y$ is generated as a Rademacher random variable. The input $\Xb$ is of the form $[\xb^{(1)}, \xb^{(2)}, \dots, \xb^{(P)}] \in\RR^{d\times P}$ with each column denoting a patch. Given $y$, a non-empty index set $\cS\subsetneqq [P]$ is randomly generated. Then for all $p\in [P]$, if $p\in \cS$, the patch $\xb^{(p)}$ is assigned as a vector uniformly chosen from $\cU_{y}$. Otherwise, $\xb^{(p)}$ is assigned as a noise vector generated from $N(\mathbf{0}, \sigma_{\mathrm{noise}}^2 \cdot( \Ib - \sum_{\bmu\in \cU_{+1}\cup \ \cU_{-1}} \bmu\bmu^\top))$.

\end{definition}

In the data model $\cD$ defined above, we consider the sets $\cU_{+1}$ and $\cU_{-1}$ following the intuition that they should consist of ``basis vectors'' that can form objects belonging to classes $+1$ and $-1$ respectively. We assume that $|\cU_{+1}| = |\cU_{-1}| = K$ for the ease of the study of ``ranks'', as it is evident that the rank of the clean data should be $2K$. Importantly,  the data model is flexible regarding the distribution of the index set $\cS$, and all our results hold as long as $\cS$ is nonempty and $\cS \subsetneqq [P]$. This allows the model to better match the nature of image data, where patches with particular spatial relationships can together form an object. 
Besides, our definition of the noise vector distribution $N(\mathbf{0}, \sigma_{\mathrm{noise}}^2 \cdot( \Ib - \sum_{\bmu\in \cU_{+1}\cup \ \cU_{-1}} \bmu\bmu^\top))$ ensures that the Gaussian noises are generated from the orthogonal complement of $\cU_{+1} \cup \ \cU_{-1}$. By separating the ``object patches'' with $p\in \cS$ and the ``background patches'' with $p\in \cS^c$, the data model only allows the noises to be added to the backgrounds of the images, which aligns with our experiment setup that produces Figure~\ref{fig:intro}.
Moreover, as a data model for a supervised training task, Definition~\ref{def:data} ensures that object patches are drawn from $\cU_{y}$ and are therefore related to the label $y$, while background patches are independent of $y$. Consequently, this data model accurately captures the nature of image classification tasks, where only the actual object in the image is directly related to the class of the image.
 
 Note that similar data models have been considered for the studies of a variety of different topics, including knowledge distillation~\citep{allen2022towards}, algorithmic biases \citep{zou2023understanding,lu2023benign}, and the ``benign-overfitting'' phenomenon~\citep{cao2022benign, kou2023benign}. Our data model in Definition~\ref{def:data} is more general and covers most of the models studied in these existing works.
 

We consider training CNNs using a dataset $\{ (\Xb_i,y_i)\}_{i=1}^n$ for binary classification, where each training data pair $ (\Xb_i,y_i)$ is generated independently from the distribution given in Definition~\ref{def:data}. To calculate the data rank, we consolidate all training inputs into a single matrix, as $\hat \Xb =[\xb_1^{(1)}, \cdots, \xb_1^{(P)},\xb_2^{(1)}, \cdots, \xb_2^{(P)}, \cdots, \xb_n^{(1)}, \cdots, \xb_n^{(P)}] \in \RR^{d\times nP}$. Then the following proposition theoretically demonstrates that the stable rank of the data matrix $\hat \Xb$ will explode at the initial stages of noise level  $\sigma_{\mathrm{noise}}$ growth.

\begin{proposition}\label{prop:data_rank}
Suppose that $d \geq \tilde\Omega(n^4)$, and $K, P\leq O(1)$. For any positive $\delta$ satisfying $\log(1/\delta)\leq O(d)$, the following inequalities concerning the stable rank of $\hat \Xb$ hold with probability at least $1-\delta$.

\begin{itemize}[leftmargin = *]
    \item If $ \sigma_{\mathrm{noise}}\sqrt{d} \leq  O(1)$, then 
    \begin{align*}
      2K -  O\bigg(\sqrt{\frac{\log(1/\delta)}{n}}\bigg) \leq   \mathrm{StableRank}(\hat \Xb) \leq 2K + O\big(\sigma_{\mathrm{noise}}^2d\big).
    \end{align*}
    \item If $ \sigma_{\mathrm{noise}}\sqrt{d} \geq  \Omega(1)$, then 
    \begin{align*}
       \mathrm{StableRank}(\hat \Xb) \geq \Big(2K + \Omega\big(\sigma_{\mathrm{noise}}^2d\big) \Big)\wedge n.
    \end{align*}
\end{itemize}
\end{proposition}

Proposition \ref{prop:data_rank} characterizes the stable rank of the data matrix $\hat \Xb$ across different noise levels $\sigma_{\mathrm{noise}}$. Specifically, it identifies a boundary for the noise level $\sigma_{\mathrm{noise}}\sqrt{d}$, where the stable rank of the data matrix exhibits significantly different patterns on either side of this boundary. Here, we consider the value of $\sigma_{\mathrm{noise}}\sqrt{d}$ rather than $\sigma_{\mathrm{noise}}$, as $\sigma_{\mathrm{noise}}$ merely represents the scale at each pixel of background patches.
By multiplying the factor $\sqrt{d}$, it correctly reflects the overall noise level of each background patch. When $\sigma_{\mathrm{noise}}\sqrt{d}\leq O(1)$, the stable rank of $\hat \Xb$ is close to the clean data rank $2K$, with a difference at most $ O(1)$. While once the noise level expresses the boundary such that $\sigma_{\mathrm{noise}}\sqrt{d} \geq  \Omega(1)$, the stable rank of $\hat \Xb$ will be greater than $2K$, exceeding it by an amount at least $\Omega(1)$.



\noindent\textbf{Two-layer CNNs.} We consider a two-layer convolutional neural network whose filters are applied to the $P$ patches $\xb^{(1)}, \xb^{(2)}, \ldots, \xb^{(P)}$, 
and the second layer parameters of the network are fixed as $+1/m$'s and $-1/m$'s. The network can be written as 
\begin{align}\label{eq:two-layer_CNNs}
    f(\Wb, \Xb) = F_{+1}(\Wb_{+1},\Xb) - F_{-1}(\Wb_{-1},\Xb).
\end{align} 
Here, $F_{+1}(\Wb_{+1},\Xb)$,  $F_{-1}(\Wb_{-1},\Xb)$ are defined as 
\begin{align*}
F_j(\Wb_j,\Xb) &= \frac{1}{m}{\sum_{r=1}^m}\sum_{p=1}^P \sigma(\la\wb_{j,r},\xb^{(p)}\ra),
\end{align*}
where $j\in \{+ 1, -1\}$, and 
$m$ is the number of convolutional filters in $F_{+1}$ and $F_{-1}$. We consider Huberized ReLU activation function $\sigma(\cdot)$ defined as follows:
\begin{align}\label{eq:def_huberized_relu}
   \sigma(z)=\left\{
\begin{aligned}
&0, &&\text{if} \ z\leq 0;\\
&\frac{z^q}{q \kappa^{q-1}}, &&\text{if} \ z\in [0,\kappa];\\
&z -\bigg(1-\frac{1}{q}\bigg)\kappa, &&\text{if} \ z\geq \kappa,
\end{aligned}
\right. 
\end{align}
where $\kappa, q$ are both absolute constants that do not rely on the training size $n$, dimension $d$, and the number of filters $m$, with $q \geq 3$.
We use $\wb_{j,r}\in\RR^{d}$ to denote the weight for the $r$-th filter (i.e., neuron), and $\Wb_{j}$ is the collection of model weights associated with $F_j$. We also use $\Wb$ to denote the collection of all model weights. We note that our CNN model can also be viewed as a CNN with average global pooling.

\noindent\textbf{Training algorithm.}
 We train the above CNN model by minimizing the empirical loss on training data set $\{(\Xb_i, y_i)\}_{i=1}^n$ with cross-entropy loss function. Specifically, we define the training loss function as follows:
\begin{align*}
    L_S(\Wb) 
    &= \frac{1}{n} {\sum_{i=1}^n} \ell[ y_i \cdot f(\Wb,\Xb_i) ],
\end{align*}
where $\ell(z) = \log(1 + \exp(-z))$ is the cross-entropy loss function for binary classification.

We consider gradient descent starting from Gaussian initialization, where each entry of $\Wb_{+1}$ and $\Wb_{-1}$ is sampled from a Gaussian distribution $N(0 , \sigma_0^2)$, where $\sigma_0^2$ is the variance. The gradient descent update of the filters in the CNN can be written as
\begin{align}
    \wb_{j,r}^{(t+1)} &= \wb_{j,r}^{(t)} - \eta \cdot \nabla_{\wb_{j,r}} L_S(\Wb^{(t)}) \label{eq:gdupdate}
    %
\end{align}
for $j \in \{\pm 1\}$ and $r \in [m]$.

\section{Stable rank of CNNs}
In the previous section, we introduce some preliminaries of this paper: we define ``stable rank," the main focus of our study, and demonstrate in Proposition~\ref{prop:data_rank} that the stable rank of the data significantly increases during the initial stages of noise level growth. In this section, we present our main results regarding the stable rank of the CNN filters: compared to that of the data, the stable rank of the CNN filters is more robust to noise—it remains close to the rank of clean data, $2K$, for a wide range of varying noise level.

Before we demonstrate our results, we first present some necessary conditions regarding the parameters of our training, including the dimension $d$, the training sample size $n$, neural network width $m$ (number of filters), learning rate $\eta$, and initialization scale $\sigma_0$.


\begin{condition}\label{condition:d_sigma0_eta}
Suppose that 
\begin{enumerate}[leftmargin = *,nosep]
    \item Dimension $d$ is sufficiently large:  $d = \tilde{\Omega}(m^4 \vee n^4)$
    \item Training sample size $n$ and neural network width $m$ satisfy $n,m = \Omega(\polylog(d))$.
    \item The learning rate $\eta$ satisfies $\eta \leq  \tilde{O}(1\wedge \sigma_{\mathrm{noise}}^{-2}d^{-1})$.
    \item The standard deviation of Gaussian initialization $\sigma_0$ is small: 
    $\sigma_{0} \leq \tilde{O}\big(d^{-1/2}\wedge\sigma_{\mathrm{noise}}^{-1}d^{-1}\big)$.
\end{enumerate}
\end{condition}
These assumptions are widely made in a series of recent works on the benign overfitting phenomena of gradient descent in learning over-parameterized CNN models \citep{chatterji2021finite,kou2023benign,cao2022benign}.
In what follows, we briefly explain the reasons and intuition behind these assumptions.
Specifically, the neural network width satisfies $m = \Omega(\polylog(d))$ to ensure that, with a sufficiently high probability, the model training will neither collapse nor vanish along any basis vector during initialization (i.e., at least one filter is activated for each $\bmu$ in $\cU_{+1}\cup \ \cU_{-1}$). Additionally, the learning rate $\eta$ is required to be small enough to guarantee the convergence of the gradient descent iterations of the training loss. And the condition on the initialization scaling $\sigma_0$ ensures that gradient descent performs feature learning rather than learning random kernels. Finally, we further assume that the number of basis vectors in both $\cU_{+1}$ and $\cU_{-1}$, and the number of patches in each input is at a constant order, i.e. $K, P\leq O(1)$. These conditions remain consistent with the Proposition~\ref{prop:data_rank}.



Now we are ready to deliver our main theorem, which characterizes the stable rank of the CNN filters $\Wb^{(T)}$ after a certain number of iterations, $T$. Besides, we also present convergence results for both training loss $L_S(\Wb^{(T)})$ and test loss $L_{\cD}(\Wb^{(T)})$, where the test loss $L_{\cD}(\Wb)$ is defined as $L_{\cD}(\Wb) =\EE_{(\Xb, y)\sim \cD}[\ell(y\cdot f(\Wb, \Xb))]$.

\begin{theorem}\label{thm:main_result}
    Suppose that Condition~\ref{condition:d_sigma0_eta} holds, and further assume that $\sigma_{\mathrm{noise}}\sqrt{d} \leq \tilde O(n^{1/q})$. Then at any iteration $T = \eta^{-1}\poly(\sigma_{\mathrm{noise}}^{-2}d^{-1},\sigma_{0}^{-1}, n, m, d)\geq \tilde\Omega\big(\frac{m}{\eta \sigma_0^{q-2}}\big)$, with probability at least $1-O(m^{-1})$, the following conclusions hold.
    \begin{enumerate}[leftmargin = *]
    \item The stable rank of CNN filters is close to the rank of clean data, $2K$:
    \begin{align}\label{eq:main_result_1}
        \Big|\sr(\Wb^{(T)}) -2K\Big|\leq O\bigg(\frac{1}{\log T}\bigg).
    \end{align}
    \item The training loss converges at a rate proportional to the reciprocal of the iterations: 
    \begin{align}\label{eq:main_result_2}
        L_S(\Wb^{(T)}) \leq O\bigg(\frac{m^2}{\eta T}\bigg).
    \end{align}
    Additionally, this convergence rate is tight: if the number of object patches of each input $\Xb_i$, i.e. $|\cS_i|$, is $1$, then it holds that $L_S(\Wb^{(T)}) = \Theta\big(\frac{m^2}{\eta T}\big)$.
    \item The CNN trained by gradient descent achieves small test loss:
    \begin{align}\label{eq:main_result_3}
        L_\cD(\Wb^{(T)}) \leq  O\bigg(\frac{m^2}{\eta T}\bigg) + \exp\big(-\tilde \Omega(d)\big).
    \end{align}
\end{enumerate}
\end{theorem}
The proof of Theorem~\ref{thm:main_result} is given in Section~\ref{section:overview_proof}. The first conclusion in 
Theorem~\ref{thm:main_result} demonstrates that, within a polynomial number of iterations, the CNN filters will reach a stable rank that is close to $2K$, which is the rank of the clean data. 
Notably, this result is established under the condition that $\sigma_{\mathrm{noise}}\sqrt{d} \leq \tilde{O}(n^{1/q})$. In comparison, by Proposition~\ref{prop:data_rank}, the stable rank of the data matrix is only close to $2K$ under the condition $\sigma_{\mathrm{noise}}\sqrt{d} \leq O (1)$. A more illustrative comparison between the stable ranks of the CNN filters and the data matrix is given in Figure~\ref{fig:comparison}. These comparisons clearly demonstrate the capability of CNNs to robustly capture the rank of clean data in the presence of relatively large background noises.

The second and third conclusions in Theorem~\ref{thm:main_result} show that gradient descent can successfully train the CNN model to achieve small training and test loss. These two conclusions back up our analysis on the stable rank by demonstrating that our analysis is built upon a proper setup where CNN training is successful. We would like to remark that establishing these theoretical guarantees on the training and test losses is particularly challenging. First of all, due to the complicated nature of the CNN model, the training objective function~\eqref{eq:two-layer_CNNs} is highly non-convex, and therefore it is challenging to show that gradient descent can minimize the training loss arbitrarily close to zero, which implies the convergence to a \textit{global minimum}. In addition, Theorem~\ref{thm:main_result} is established in a high-dimensional setting where the number of parameters ($2md$) is significantly larger than the sample size $n$. In this context, the CNN will eventually \textit{overfit} the training data (achieving a training loss arbitrarily close to zero, as shown in the second conclusion), making it challenging to establish strong guarantees on the test loss.




\begin{figure}[ht!]
	\begin{center}
 		\hspace{-0.2in}
            \includegraphics[width=0.55\textwidth]{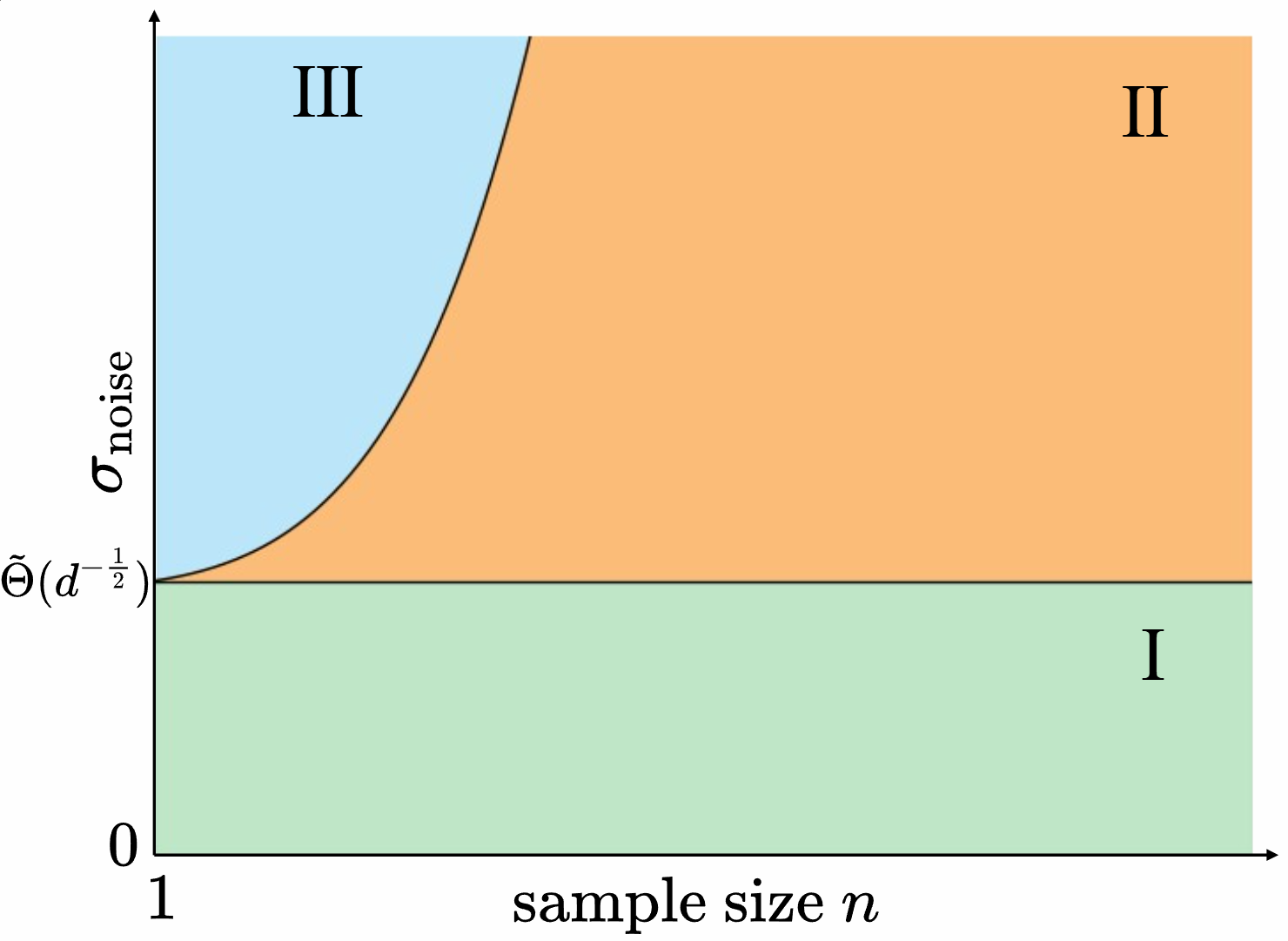}
	\end{center}
	\vskip -12pt
        \caption{Illustration of the stable ranks of the data matrix and CNN filters under different noise levels $\sigma_{\mathrm{noise}}$ and sample sizes $n$. In Region I, stable ranks of both the data matrix and the CNN filters stay close to the rank of the clean data. In Region II, the stable rank of CNN remains close to the rank of the clean data, while the stable rank of the data matrix explodes. In Region III, the stable ranks of both the data matrix and CNN filters explode. It is evident that the stable rank of CNN remains close to the rank of the clean data under a much wider regime, demonstrating its robustness to background noises.}
	\label{fig:comparison}
\end{figure}

\section{CNNs learn the clean-data subspace}
In the previous section, we present our main results regarding the stable rank of CNN filters: the stable rank of CNN filters approaches the rank of the clean data. A natural conjecture to explain this phenomenon is that the filters successfully learn the structure of the clean data, namely the basis vectors associated with object patches. This also accounts for why two-layer CNNs can achieve small training and test losses simultaneously. To rigorously verify this conjecture, we carefully examine the components of each filter throughout the training process. Our findings reveal there exist two distinct patterns among the filters: some significantly learn a particular basis vector, while others remain close to their initial state. To better present our theoretical findings, we denote $\cN$ as the orthogonal complement to the subspace spanned by the basis vectors in $\cU_{+1}\cup \ \cU_{-1}$, i.e. $\cN = \mathrm{span}(\cU_{+1}\cup \ \cU_{-1})^{\perp}$. We also denote $\Pb_{\cN}$ as the projection matrix of the subspace $\cN$, i.e. $\Pb_{\cN} = \Ib_d - \sum_{y\in\{\pm 1\}}\sum_{k=1}^K\bmu_{y, k}\bmu_{y, k}^\top$. Then based on these notations, we provide a formal illustration of our theoretical findings in the following theorem. 

\begin{theorem}\label{thm:main_result3}
Denote $\cU_{+1} = \{\bmu_{1,1}, \bmu_{1,2}, \ldots, \bmu_{1,K}\}$ and $\cU_{-1} = \{\bmu_{-1,1}, \bmu_{-1,2}, \ldots, \bmu_{-1,K}\}$.
    Under the same conditions as Theorem~\ref{thm:main_result}, there exist $K$ distinct filters $\wb_{1, r_{1, 1}}, \cdots, \wb_{1, r_{1, K}}$ in $\Wb_{+1}$ corresponding to $\bmu_{1, 1}, \cdots, \bmu_{1, K}$, and $K$ distinct filters $\wb_{-1, r_{-1, k}}, \cdots, \wb_{-1, r_{-1, K}}$ in $\Wb_{-1}$, corresponding to $\bmu_{-1, 1}, \cdots, \bmu_{-1, K}$, respectively. At any large iteration $T = \eta^{-1}\poly(d^{-1}\sigma_{\mathrm{noise}}^{-2}, \sigma_{0}^{-1}, n, m, d)\geq \tilde\Omega\big(\frac{m}{\eta \sigma_0^{q-2}}\big)$, 
    the following results hold for all $j\in \{\pm 1\}$ and $k\in [K]$, with probability at least $1-O(m^{-1})$:
    \begin{align*}
        &\Big\|\wb_{j, r_{j, k}}^{(T)} - \big(m\log (T)\big)\bmu_{j, k}\Big\|_2 \leq O\big(m\log(m\vee \eta^{-1})\big);\\
        &\big\|\wb_{j, r}^{(T)}\big\|_2\leq O(m^{-1}) + \tilde O(\sigma_0d^{1/2}), \enspace \text{if} \enspace r\neq r_{j, k}.
    \end{align*}
    Moreover, $\big\|\Pb_{\cN}\cdot ( \wb_{j, r}^{(T)} - \wb_{j, r}^{(0)} )\big\|_2\leq O(\sigma_0 n^{1/2})$ for all $j\in\{-1, +1\}$ and $r\in [m]$.
\end{theorem}
Theorem~\ref{thm:main_result3} demonstrate that for each basis vector $\bmu_{j, k}$, there exists exactly one distinct filter $\wb_{j, r_{j,k}}$ can significantly learn this basis vector. Specifically, Theorem~\ref{thm:main_result3} straightforwardly implies that when $T$ is large, the length of the projection of $\wb_{j, r_{j, r}}^{(t)}$ onto $\bmu_{j, k}$, i.e., $\la \wb_{j, r_{j, r}}^{(t)},\bmu_{j, k}\ra$, is bounded from both above and below at the same order of $m\log (T)$. This indicates that we can accurately characterize the growth of $\wb_{j, r_{j, r}}^{(t)}$ along the direction of $\bmu_{j ,k}$. For those other filters not corresponding to any basis vector, their norm is upper bounded by a small value, suggesting that they remain close to their initialization throughout the training process. Additionally, for all filters, their projection onto the subspace $\cN$ remains at an exceptionally low level. This implies that all filters, including those that can increase significantly, learn very little from the background noises. Therefore, Theorem~\ref{thm:main_result3} elaborates on the properties of CNNs trained by gradient descent in more details, and provides supplementary results that further explain why the stable rank results of the CNN holds as stated in Theorem~\ref{thm:main_result}.


According to the conclusions of Theorem~\ref{thm:main_result3}, we can intuitively treat $\wb_{j, r_{j, k}}^{(T)}\approx m\log(T)\bmu_{j, k}$ and $\wb_{j, r}^{(T)}\approx \mathbf{0}$ if $r\neq r_{j, k}$. One interesting finding is that the coefficient $m\log (T)$ remains consistent across different basis vectors.  Based on this observation, if we rearrange the rows in $\Wb_j$ such that the first $K$ rows exactly represent those $K$ filters corresponding to different basis vectors from $\cU_{j}$ in order, the permutated filter matrix will align directionally with the matrix $\Wb_{j}^* = [\bmu_{j, 1}, \cdots, \bmu_{j, K} , \mathbf{0}_d , \cdots, \mathbf{0}_d]^\top$. The following corollary provides a formal illustration of this conclusion.

\begin{corollary}\label{crlry:main_result}
    Let $\Wb_{j}^* = [\bmu_{j, 1}, \cdots, \bmu_{j, K} , \mathbf{0}_d , \cdots, \mathbf{0}_d]^\top \in \RR^{m\times d}$ with $j\in \{\pm 1\}$. Under the same conditions as Theorem~\ref{thm:main_result}, there exist permutation matrices $\Pb_+ ,\Pb_-\in \RR^{m \times d}$. At any large iteration 
    $T = \eta^{-1}\poly(\sigma_{\mathrm{noise}}^{-2}d^{-1},\sigma_{0}^{-1}, n, m, d)\geq \tilde\Omega\big(\frac{m}{\eta \sigma_0^{q-2} }\big)$, with probability at least $1-O(m^{-1})$, it holds that
    \begin{align*}
        \Bigg\|\frac{\Wb_{+1}^*}{\| \Wb_{+1}^* \|_F}- \Pb_+\frac{\Wb_{+1}^{(T)}}{ \| \Wb_{+1}^{(T)} \|_F} \Bigg\|_F, \Bigg\|\frac{\Wb_{-1}^*}{\| \Wb_{-1}^* \|_F}- \Pb_{-}\frac{\Wb_{-1}^{(T)}}{ \| \Wb_{-1}^{(T)} \|_F} \Bigg\|_F \leq O\bigg(\frac{1}{\log (T)}\bigg).
    \end{align*}
\end{corollary}
Corollary~\ref{crlry:main_result} demonstrates that as the training proceeds, the normalized filter matrix will approach the normalized clean data matrix at a rate of $O\big(\frac{1}{\log (T)}\big)$. This corollary intuitively explains why the stable rank of CNN filters $\Wb^{(T)}$ approaches $2K$. Based on the directional convergence results, we can approximately take $\Wb_{+1}^{(T)} \approx \frac{\| \Wb_{+1}^{(T)} \|_F}{\| \Wb_{+1}^* \|_F}\Pb_+^{-1}\Wb_{+1}^*$. Multiplication by an orthogonal matrix or a constant does not affect the stable rank, as the stable rank is determined by the relative scales of different singular values. Since the stable rank of $\Wb_{+1}^*$ is $K$, we can conclude the stable rank of $\Wb_{+1}^{(T)}$ will approach $K$, at the same rate as their directional convergence. Similar conclusion also holds for $\Wb_{-1}^{(T)}$. Additionally, as the components of $\Wb_{+1}^{(T)}$ and $\Wb_{-1}^{(T)}$ are nearly orthogonal based on our assumption that $\mathrm{span}(\cU_{+1})$ and $\mathrm{span}(\cU_{-1})$ are orthogonal, it is evident that the stable rank of $\Wb^{(T)}$ is approximately $2K$.
\section{Experiments}
In this section, we present our experimental results to
backup our theoretical results and show a two-layer CNN is robust to background noise in data.

We generate training data from the MNIST \citep{MNIST} and CIFAR10 \citep{CIFAR} datasets according to Definition \ref{def:data}. 
We use images from two selected classes as the source of object patches for the $y=-1$ class and the $y=1$ class, 
respectively. 
To convert the original image into a data point with a low-rank structure, we reshape each image into a vector and stack all these vectors into a matrix. Then, we use PCA to control the number of principal components in this matrix, reflecting the rank of the clean data. After performing PCA, we reshape each column back into an image and pad it with a circle of background patches filled with Gaussian noise. Specifically, each pixel in the background patches is sampled from $N(0, \sigma_{\mathrm{noise}}^2)$, 
where we set $\sigma_{\mathrm{noise}}$ to different values to verify how our model behaves under varying levels of noise. 
We consider a CNN model as defined in Section \ref{sec:ProblemSetup}, and the weights of this CNN model are initialized from Gaussian distribution with a small standard deviation, consistent with our theoretical settings.
For different data, we use different setups to generate the noise data and run the full batch gradient decent to train the CNN:

\begin{figure}[!tbp]
    \centering
    \subfloat[$\sigma_{\mathrm{noise}} = 0.01$]{
        \includegraphics[width=0.25\linewidth]{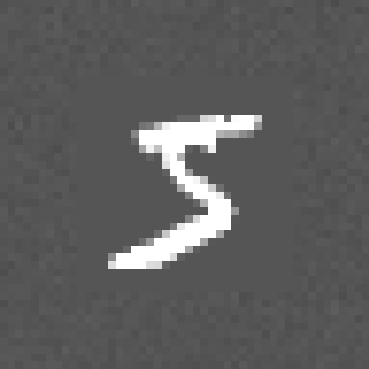}
    }
    \subfloat[$\sigma_{\mathrm{noise}} = 0.1$]{
        \includegraphics[width=0.25\linewidth]{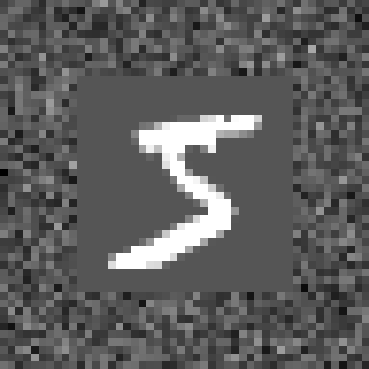}
    }
    \subfloat[$\sigma_{\mathrm{noise}} = 0.2$]{
        \includegraphics[width=0.25\linewidth]{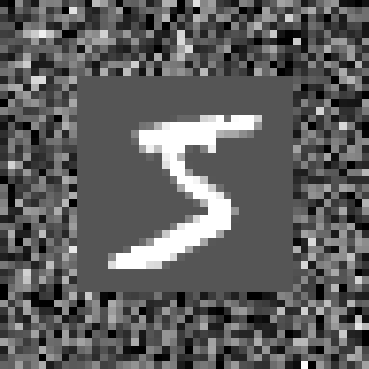}
    }
    \caption{
        Illustration of a training image from the MNIST dataset, reduced to rank $10$ and padded with a circle of noise.
    }
    \label{fig:MNISTexample}
\end{figure}

\noindent\textbf{MNIST.}
The MNIST images undergo dimensionality reduction to three levels of rank: $10$, $20$, and $30$. 
Then each column of the matrix, corresponding to an image, is reshaped to its original size and 
padded with a 14-pixel wide circle of noise (An example is shown in Figure \ref{fig:MNISTexample}). 
The padded pixels are entry-wise Gaussian noise $N(0, \sigma_{\mathrm{noise}}^2)$, where $\sigma_{\mathrm{noise}}$ is set to $0$, $0.01$, $0.1$, $0.12$, $0.15$, $0.18$, and $0.2$. 
Additionally, the model width $m$ is set to 128. For each rank level, the standard deviation of initialization distribution is set to $1\text{e-3}$, $1\text{e-3}$, and $1\text{e-2}$, respectively.

\noindent\textbf{CIFAR10.}
To reduce the complexity of the CIFAR-10 dataset and facilitate the observation of the phenomenon, we transform the original CIFAR-10 images into embeddings using ResNet-18. All embeddings are then stacked into a matrix, which undergoes dimensionality reduction to ranks of $15$, $20$, and $25$. After reducing the dimensions of the embedding matrix, we concatenate each embedding with a noise vector.
Here, the standard deviation $\sigma_{\mathrm{noise}}$ is set to $0$, $0.1$, $0.5$, $0.62$, $0.65$, $0.7$ and $0.8$. 
For the hyperparameters of the CNN model, the width $m$ is set to $256$, $512$ and $128$, and the standard deviation of initialization distribution are set to $1\text{e-6}$, $3\text{e-7}$, $3\text{e-7}$ for each different rank level.

\noindent\textbf{Synthetic Data.} 
In addition to using two real-world datasets, we also conduct experiments on synthetic data. 
We strictly follow Definition \ref{def:data} to generate the synthetic data. 
For the object patches, we set $K = 10, 20$, and $30$, and choose one-hot vectors as the basis vectors assigned to object patches. Then, we set 
$|\cS| = 1$ and $P = 3$, which means each data instance contains one object patch and two Gaussian noise patches. 
And $\sigma_{\mathrm{noise}}$ is set to $0$, $0.001$, $0.0065$, $0.009$, $0.01$, $0.012$, $0.015$. 
For the CNN model,  the standard deviation of initialization distribution is set to $1 \times 10^{-4}$, 
and the model width $m$ is set to 128.

\noindent\textbf{Result.} 
According to Theorem \ref{thm:main_result3}, the rank of the filter 
is approximately equal to the number of basis vectors. In this experiment, we report 
three different ranks: the dimension of PCA, \textit{i.e.} the rank of pure basis vectors without noise, the rank of model weights, and the rank 
of the data with noise. Here, we verify whether the dimension of PCA is roughly 
equal to the rank of the filters after training.
To evaluate the rank of the model weights and the noise data, we denote the number of 
their singular values larger than $\lambda_{\max} / 100$ as the rank, where $\lambda_{\max}$ 
is the maximal singular value of the corresponding matrix. 
In all experiments, the rank results are presented when models have been trained to achieve 
a very small training error, in the range of $1\text{e-1}$ to $1\text{e-2}$.
As shown in Figure \ref{fig:MNISTandCIFAR}, it is evident that the rank of data will dramatically explode as the level of background noise starts to increase. In comparison,  the rank of the 
CNN filter remains approximately the same as the PCA dimension for all different levels of background noise. These empirical observations support our theoretical findings that CNNs can recover the intrinsic low-rank structure of the clean data, even when significant noise has been added to the data, which obscures the low-rank structure in the data matrix.


\begin{figure}[!tbp]
    \centering
    \subfloat[Rank visualization of data and CNN filters on MNIST dataset.]{
        \includegraphics[width=0.95\linewidth]{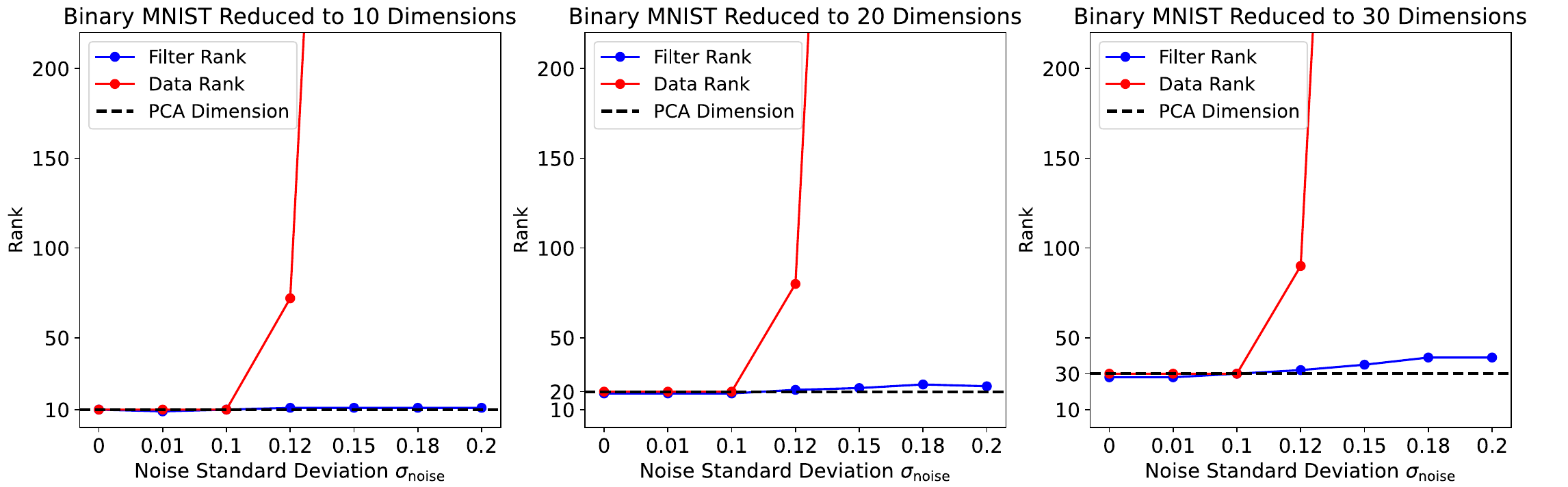}
    }\\
    \subfloat[Rank visualization of data and CNN filters on CIFAR10 dataset.]{
        \includegraphics[width=0.95\linewidth]{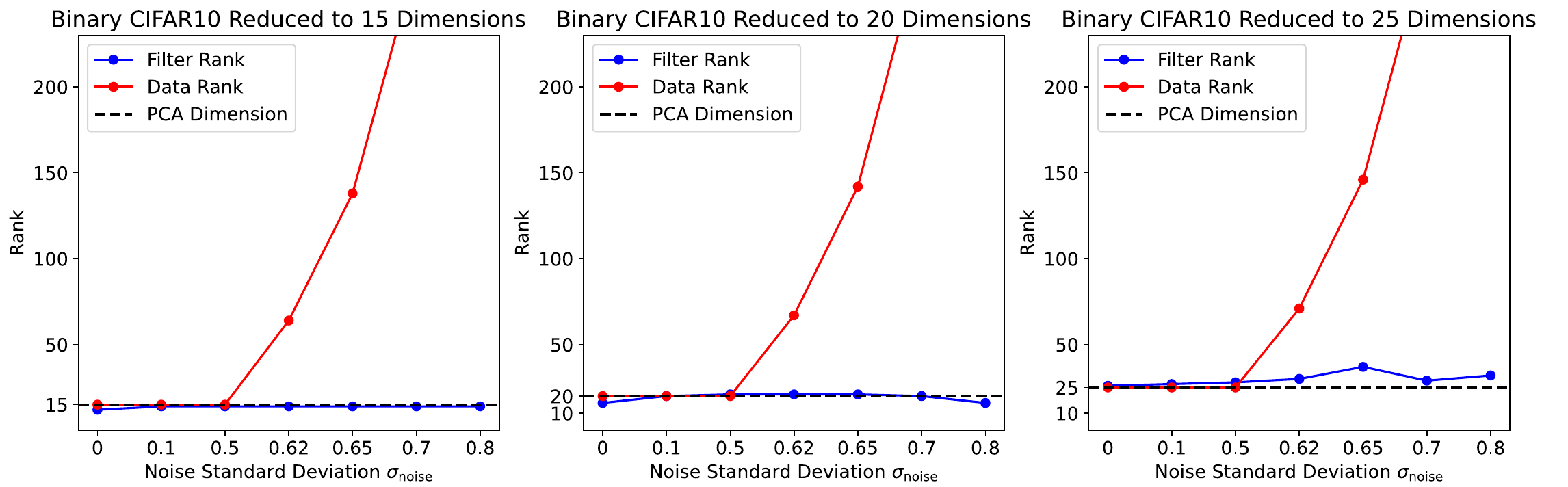}
    }\\
    \subfloat[Rank visualization of data and CNN filters on Synthetic Data.]{
        \includegraphics[width=0.95\linewidth]{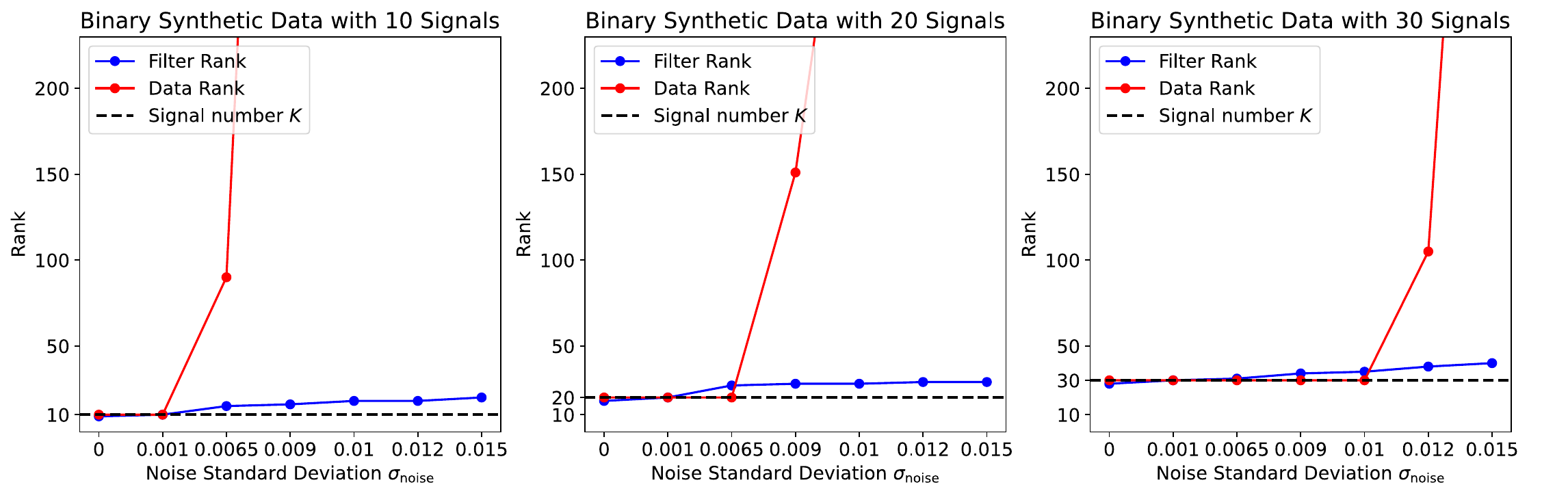}
    }
    \caption{
        Rank of the data and learned filters under different noise levels. Here $x$-axis represents the value of the noise level $\sigma_{\mathrm{noise}}$, and $y$-axis is the rank. From the figures, it can be clearly observed that the data rank increases rapidly as the noise becomes stronger, while the rank of the CNN filters remains robust against the noise, and keeps being the same as the rank of clean data.
    }
    \label{fig:MNISTandCIFAR}
\end{figure}

\section{Overview of proof technique}\label{section:overview_proof}
In this section, we explain how we establish our main theoretical results. We begin by introducing several key lemmas that play a crucial role in proving Theorem~\ref{thm:main_result3}. We then present a proof for Theorem~\ref{thm:main_result3} based on these lemmas. Finally, we demonstrate the proof for Theorem~\ref{thm:main_result}, which can be easily derived from the results of Theorem~\ref{thm:main_result3}.
\subsection{Proof of Theorem~\ref{thm:main_result3}}

We first introduce some further notations regarding the data model $\cD$ defined in Definition~\ref{def:data}. For each input $\Xb_i$ with $i \in [n]$, $\cS_i$ determines the indices of the object patches among this input, with $s_i=|\cS_i|$. We denote $\bnu_{i, 1}, \cdots, \bnu_{i, s_i}\in \cU_{y_i}$ as the basis vectors in the $i$-th training data point. 
Additionally, we denote $\bxi_{i, 1}, \bxi_{i, 2}, \cdots, \bxi_{i, P-s_i}$ as the Gaussian noise vectors assigned to the background patches, i.e., $\bxi_{i, k'}\sim N\big(\mathbf{0}, \sigma_{\mathrm{noise}}^2 \cdot (\Ib_d - \sum_{y\in\{\pm 1\}}\sum_{k=1}^K\bmu_{y, k}\bmu_{y, k}^\top)\big)$. Based on these notations, 
we can rewrite our gradient descent iterative formula~\eqref{eq:gdupdate} as
\begin{align}\label{eq:gdupdate3}
    \wb_{j,r}^{(t+1)} &= \wb_{j,r}^{(t)} - \eta \cdot \nabla_{\wb_{j,r}} L_S(\Wb^{(t)}) \nonumber\\
    &= \wb_{j,r}^{(t)} - \frac{\eta}{nm} \sum_{i=1}^n\sum_{p=1}^{P} \ell_i'^{(t)} \cdot  \sigma'(\la\wb_{j,r}^{(t)}, \xb_{i}^{(p)}\ra)\cdot j y_{i}\xb_{i}^{(p)}\nonumber\\
    &= \wb_{j,r}^{(t)} - \frac{\eta}{nm} \sum_{i=1}^n\sum_{k'=1}^{P-s_i} \ell_i'^{(t)} \cdot  \sigma'(\la\wb_{j,r}^{(t)}, \bxi_{i, k'}\ra)\cdot j y_{i}\bxi_{i, k'}\nonumber\\
    & \quad - \frac{\eta}{nm} \sum_{i=1}^n\sum_{k=1}^{s_i} \ell_i'^{(t)} \cdot \sigma'(\la\wb_{j,r}^{(t)}, \bnu_{i, k}\ra)\cdot jy_i \bnu_{i, k}, 
\end{align}
where $\ell_i'^{(t)} = \ell'[ y_i \cdot f(\Wb^{(t)},\Xb_i) ] $. 
This iterative rule reveals that the update of all filters $\wb_{j, r}$'s are always a linear combination of the basis vectors from $\cU_{+1}$, $\cU_{-1}$, and Gaussian noise vectors $\bxi_{i, k'}$'s generated in the subspace $\cN$, which is the complement subspace of $\cU_{+1} \cup \ \cU_{-1}$. Based on these observations, to characterize the content learned by each filter, it is sufficient to study the inner products between this filter and the basis vectors and noise vectors. For each basis vector $\bmu_{y, k}$ from $\cU_{y}$, we provide a careful analysis of the dynamic process of its inner product with each filter $\wb_{j, r}$. As to demonstrate our desired conclusions regarding stable rank, we need a refined quantification of the projection of each filter onto each basis vector. While for noise vectors, we do not consider the inner product with each filter. Instead, we focus on the projection of each filter $\wb_{j, r}$ into the subspace $\cN$, i.e. $\Pb_{\cN}\cdot\wb_{j, r}^{(t)}$. We propose this design as a small upper-bound for $\big\|\Pb_{\cN}\cdot\wb_{j, r}^{(t)}\big\|_2$ is sufficient to illustrate that the filter $\wb_{j, r}$ learns little from any noise vectors.

\noindent\textbf{When training starts, filters exhibit distinct patterns according to initialization.}
The working mechanism of CNNs, in which each filter interacts with all filters, implies that these filters exhibit symmetry and translation invariance. These properties also hold when the filters are updated using gradient descent. Specifically, we note that the iterative rule~\eqref{eq:gdupdate3} is nearly the same for different filters, differing only in their initialization along various directions, i.e. $\la\wb_{j,r }^{(0)}, \bmu_{y, k}\ra$'s and $\la\wb_{j,r }^{(0)}, \bxi_{i, k'}\ra$'s. Based on this observation, we intuitively conjecture that the initialization of the inner products $\la\wb_{j, r}^{(0)}, \bmu_{y, k}\ra$ might play a key role in analyzing the increasing trajectory of $\la\wb_{j, r}^{(t)}, \bmu_{y, k}\ra$.  Therefore, we focus on the filters with the specific initialization, and successfully demonstrate that the filters with the maximum inner product with a basis vector at initialization will exhibit significantly distinct patterns during the early training phase. We provide a formal illustration of this phenomenon in the following lemma.
\begin{lemma}\label{lemma:phase1_main}
Under the same conditions as Theorem~\ref{thm:main_result},
for each basis vector $\bmu_{j, k}$ from $\cU_{j}$, let $\wb_{j, r_{j,k}}$ denote the filter in $\Wb_{j}$ with largest inner product with $\bmu_{j, k}$, i.e $\wb_{j, r_{j,k}} = \argmax_{\wb_{j, r}\in \Wb_{j}} \la\wb_{j, r}^{(0)}, \bmu_{j, k}\ra$. Then with probability at least $1-O(m^{-1})$, the filters $\wb_{j, r_{j,k}}$'s are distinct for different basis vector $\bmu_{j, k}$, and there exist an iteration number $T_1 = \tilde\Theta\big(\frac{m}{\eta \sigma_0^{q-2}}\big)$ such that the following results hold for all $j\in \{\pm 1\}$ and $k\in [K]$:
\begin{align*}
    &\la \wb_{j, r_{j, k}}^{(T_1)}\bmu_{j, k}\ra \geq \kappa, \enspace \text{and}\enspace \big|\la\wb_{j,r}^{(T_1)}, \bmu_{j, k}\ra \big|\leq \frac{1}{4Km}, \enspace \text{if} \enspace r\neq r_{j, k}.
\end{align*}
Moreover, $\big\|\bXi_{j, r}^{(t)}\big\|_2^2 \leq \frac{\sigma_0^2 n P}{2}$ hold for all $j\in \{\pm 1\}$, $r\in[m]$ and $0 \leq t \leq T_1$, where $\bXi_{j, r}^{(t)}=\Pb_{\cN}\cdot(\wb_{j, r}^{(t)}-\wb_{j, r}^{(0)})$.
\end{lemma}

Theorem~\ref{thm:main_result} asserts that for each $\bmu_{j, k}\in \cU_{j}$, there exists a specific filter in $\Wb_{j}$ can significantly learn this basis vector. Lemma~\ref{lemma:phase1_main} identifies this specific filter as the one with the largest inner product at initialization. Furthermore, Lemma~\ref{lemma:phase1_main} demonstrates there exist an iteration $T_1= \tilde\Theta\big(\frac{m}{\eta \sigma_0^{q-2}}\big)$. At this iteration, those filters having the largest inner product with certain basis vectors, i.e. $\wb_{j, r_{j, k}}$'s, will exhibit significantly different patterns compared to other filters: the projection of $\wb_{j, r_{j,k}}^{(T)}$ onto $\bmu_{j, k}$ attains $\kappa$, while those of other filters stay at a small value. Moreover, the projection of all filters into the subspace $\cN$ is extremely small. Since each entry of filters is initialized from standard Gaussian distribution, by the symmetry of Gaussian distribution, the same filter having the largest inner product for multiple basis vectors simultaneously can only occur with extremely small probability. Therefore, we can intuitively get the idea that $\Wb_{j}^{(T_1)}$ already exhibits a low-rank structure with rank $K$: it contains $K$ row, each approximately aligned with a distinct basis vector in $\cU_{j}$, while other rows remain close to $\mathbf{0}$.




\noindent\textbf{Matching lower and upper bounds for projections when loss converges.}
As we discussed above, the conclusions of Lemma~\ref{lemma:phase1_main} can already intuitively imply that $\Wb_{j}^{(T_1)}$ approximately exhibits a low-rank structure with rank $K$. However, to obtain an effective estimation of the stable rank, it is essential to accurately characterize all singular values, providing both lower and upper bounds. The conclusion of Lemma~\ref{lemma:phase1_main} presenting merely a lower bound $\kappa$ is not sufficient to derive a satisfactory conclusion regarding stable rank. Additionally, it is clear that a constant-order lower bound for the inner products $\la\wb_{j, r_{j, k}}^{(t)},\bmu_{j, k}\ra$'s is insufficient to guarantee the convergence of the loss, given our two-layer CNNs model~\eqref{eq:two-layer_CNNs}. Therefore, we consider the subsequent training stage of two-layer CNNs, and demonstrate a matching upper and lower bound for the increase rate of $\la\wb_{j, r_{j, k}}^{(t)},\bmu_{j, k}\ra$'s in the following lemma.

\begin{lemma}\label{lemma:phase2_main}
Under the same conditions as Theorem~\ref{thm:main_result}, and with $\wb_{j,r_{j, k}}, \bXi_{j, r}$ defined as in Lemma~\ref{lemma:phase1_main}, at any large iteration $T^{*} = T_1 + \eta^{-1}\poly( \sigma_{\mathrm{noise}}^{-2}d^{-1},\sigma_{0}^{-1}, n, m, d)$, the following results hold  for all $j\in \{\pm 1\}$ and $k\in[K]$, with probability at least $1-O(m^{-1})$:
\begin{align*}
    & \Big|\la\wb_{j,r_{j, k}}^{(T^*)} , \bmu_{j, k}\ra -m\log (T^*-T_1) -2m\log m +m\log \delta \Big|\leq O(m); \\
    &\big|\la\wb_{j,r}^{(T^*)}, \bmu_{j, k}\ra\big| \leq \frac{1}{2Km}, \enspace \text{if} \enspace r \neq r_{j, k}.
\end{align*}
Moreover, $\big\|\bXi_{j, r}^{(t)}\big\|_2^2 \leq 2\sigma_0^2 n P$ for all $j\in \{\pm 1\}$, $r\in[m]$ and $0 \leq t \leq  T^*$.
%
%
\end{lemma}

Lemma~\ref{lemma:phase2_main} illustrates that during any polynomial number of iterations exceeding $T_1$ defined in Lemma~\ref{lemma:phase1_main}, the increase of the inner product $\la\wb_{j, r_{j, k}}^{(t)}, \bmu_{j,k}\ra$ is simultaneously bounded from above and below. Moreover, both the upper and lower bounds have matching rates, increasing logarithmically with respect to $t$. Besides, the projections of other filters onto each basis vector or the subspace $\cN$, always remain bounded by some small value. As we have discussed previously, the filters $\wb_{j, r_{j, k}}$ are distinct for different basis vectors with high probability. These observations intuitively explain the conclusions that $\wb_{j, r_{j,k}}^{(T^{*})} \approx m\log (T^{*}- T_1)\bmu_{j, k}$, and $\wb_{j, r}^{(T^{*})} \approx \wb_{j, r}^{(0)}$ if $r\neq r_{j, k}$ in Theorem~\ref{thm:main_result3}.


\noindent\textbf{The filters in $\Wb_{+1}$ barely learn basis in $\cU_{-1}$ and vice versa.} In Lemmas~\ref{lemma:phase1_main} and~\ref{lemma:phase2_main}, we clearly quantify the increase of $\la\wb_{1, r}^{(t)}, \bmu_{1, k}\ra$ and $\la\wb_{-1, r}^{(t)}, \bmu_{-1, k}\ra$ through the entire training process. However, by the gradient descent iterative rule~\eqref{eq:gdupdate3}, we still need to calculate $\la\wb_{1, r}, \bmu_{-1, k}\ra$ and $\la \wb_{-1, r}, \bmu_{1, k}\ra$. The following lemma indicates that these terms are bounded by a small amount throughout the training.

\begin{lemma}\label{lemma:opposite_learning}
    Under the same conditions as Theorem~\ref{thm:main_result}, for all $t>0$, $r\in[m]$, and $k\in [K]$, with probability at least $1-O(m^{-1})$, it holds that 
    \begin{align*}
        \la\wb_{1, r}^{(t)}, \bmu_{-1, k}\ra, \la\wb_{-1, r}^{(t)}, \bmu_{1, k}\ra \leq \tilde O(\sigma_0).
    \end{align*}
\end{lemma}
Now, we are ready to prove our main Theorem~\ref{thm:main_result3}.
\begin{proof}[Proof of Theorem~\ref{thm:main_result3}]
We establish the proof for $j=1$, while the proof for $j=-1$ is totally identical.  Let $T=\eta^{-1}\poly( \sigma_{\mathrm{noise}}^{-2}d^{-1},\sigma_{0}^{-1}, n, m, d)\geq\tilde\Omega\big(\frac{m}{\eta \sigma_0^{q-2} }\big)$. For any $\bmu_{1,k}\in\cU_{+1}$ and its corresponding filter $\wb_{1, r_{1, k}}$, we can calculate the  norm of $\wb_{1, r_{1, k}}$ as
\begin{align*}
    &\big\|\wb_{1, r_{1, k}}^{(T)} - \big(m\log T\big)\cdot\bmu_{1, k}\big\|_2\\
    =& \big\|\wb_{1, r_{1, k}}^{(0)} + 
    \sum_{y\in{\pm 1}}\sum_{k'=1}^K\la\wb_{1, r_{1,k}}^{(T)}-\wb_{1, r_{1,k}}^{(0)}, \bmu_{y, k'}\ra\bmu_{y, k'} + \bXi_{1, r_{1,k}}^{(T)}- \big(m\log T\big)\cdot\bmu_{1, k}\big\|_2\\
    \leq & \big\|\wb_{1, r_{1, k}}^{(0)}\big\|_2 + \sum_{y\in{\pm 1}}\sum_{k'=1}^K\big|\la\wb_{1, r_{1,k}}^{(0)}, \bmu_{y, k'}\ra\big|+ \sum_{k'\neq k}\big|\la\wb_{1, r_{1,k}}^{(T)}, \bmu_{1, k'}\ra\big| + \sum_{k'=1}^K\big|\la\wb_{1, r_{1,k}}^{(T)}, \bmu_{-1, k'}\ra\big|\\
    & + \big|\la\wb_{1, r_{1,k}}^{(T)}, \bmu_{1, k}\ra -m\log T\big| + \big\|\bXi_{1, r_{1,k}}^{(T)}\big\|_2\\
    \leq & \tilde O(\sigma_0 d^{1/2}) + \tilde O(2K\sigma_0) + O\Big(\frac{K}{m}\Big) + \tilde O(K\sigma_0) + O\big(m\log(m\vee \eta^{-1})\big) + O(\sigma_0 n^{1/2})\\
    \leq& O\big(m\log(m\vee \eta^{-1})\big).
\end{align*}
The first inequality holds by triangle inequality. The second inequality is derived from Lemma~\ref{lemma:phase2_main}, Lemma~\ref{lemma:opposite_learning}, and concentration results in Appendix~\ref{section:concentration}, which guarantee that $\big\|\wb_{1, r}^{(0)}\big\|_2\leq \tilde O(\sigma_0 d^{1/2})$ and $\big|\la\wb_{1, r}^{(0)}, \bmu_{y, k}\ra\big|\leq \tilde O(\sigma_0)$ for all $r\in[m]$, $y\in \{\pm 1\}$ and $k\in [K]$. The last inequality follows from the scale relationships among these parameters as assumed in Condition~\ref{condition:d_sigma0_eta}.
For other filters $\wb_{1,r}$ with $r\neq r_{1, k}$ for all $k\in [K]$, we have
\begin{align*}
     \big\|\wb_{1, r}^{(T)}\big\|_2 \leq &\big\|\wb_{1, r}^{(0)}\big\|_2 + \sum_{y\in\{\pm 1\}}\sum_{k=1}^K\big|\la\wb_{1, r}^{(T)}-\wb_{1, r}^{(0)}, \bmu_{y, k}\ra\big| + \|\bXi_{1, r}^{(T)}\|_2\\
     \leq & \big\|\wb_{1, r}^{(0)}\big\|_2 + \sum_{y\in{\pm 1}}\sum_{k=1}^K\big|\la\wb_{1, r}^{(0)}, \bmu_{y, k}\ra\big| + \sum_{k=1}^K\big|\la\wb_{1, r}^{(T)}, \bmu_{1, k}\ra\big| + \sum_{k=1}^K\big|\la\wb_{1, r}^{(T)}, \bmu_{-1, k}\ra\big| + \big\|\bXi_{1, r}^{(T)}\big\|_2\\
    \leq& \tilde O(\sigma_0 d^{1/2}) + \tilde O(2K\sigma_0) + O\Big(\frac{K}{m}\Big) + \tilde O(K\sigma_0)  + O(\sigma_0 n^{1/2})\leq \tilde O(\sigma_0 d^{1/2}) +  O(m^{-1}).
\end{align*}
Similarly, the first inequality is from triangle inequality. The second inequality is from Lemma~\ref{lemma:phase2_main}, Lemma~\ref{lemma:opposite_learning}, and concentration results in Appendix~\ref{section:concentration}. And the last inequality is by Condition~\ref{condition:d_sigma0_eta}. By definition of $\Pb_{\cN}$, we can directly have $\Pb_{\cN}\cdot \big(\wb_{1,r}^{(T)}- \wb_{1,r}^{(0)}\big) =\bXi_{1, r}^{(t)}$, which proves the last conclusion in Theorem~\ref{thm:main_result3}.
\end{proof}
\subsection{Proof of Theorem~\ref{thm:main_result}}\label{section:proof_main_result}
In this section, we provide a detailed proof of our main result, Theorem~\ref{thm:main_result}, building on the preceding lemmas and conclusions. 
\begin{proof}[Proof of Theorem~\ref{thm:main_result}]
We first prove the result regarding the stable rank of $\Wb^{(T)}$. By Theorem~\ref{thm:main_result3}, we have $\big\|\wb_{j, r_{j, k}}^{(T)}\big\|_2 = m\log T - m\log(m^2\eta^{-1}) + O(m)$, and $\big\|\wb_{j, r}^{(T)}\big\|_2 \leq O(m^{-1}) + \tilde O(\sigma_0 d^{1/2})$ when $r\neq r_{j,k}$. By definition of Frobenius norm and operator norm, we can derive that 
\begin{align*}
    &\big\|\Wb^{(T)}\big\|_F =  \sqrt{\sum_{j\in\{\pm 1\}}\sum_{k=1}^{K}\big\|\wb_{j, r_{j, k}}^{(T)}\big\|_2^2 +  \sum_{j\in\{\pm 1\}}\sum_{r\neq r_{j, k}}\big\|\wb_{j, r}^{(T)}\big\|_2^2} = \sqrt{2K}m\log T - \sqrt{2K}m\log(m^2\eta^{-1}) + O(m);\\
    &\big\|\Wb^{(T)}\big\|_2 = \max_{\bmu\in \RR^d} \frac{\big\|\Wb^{(T)}\bmu\big\|_2}{\|\bmu\|_2}=\max_{\bmu\in \cU_{+1}\cup \ \cU_{-1}} \big\|\Wb^{(T)}\bmu\big\|_2= m\log T - m\log(m^2\eta^{-1}) + O(m).
\end{align*}
Then by the definition of stable rank, we have 
\begin{align*}
    &\bigg|\sr(\Wb^{(T)}) -2K\bigg| = \Bigg|\frac{\big\|\Wb^{(T)}\big\|_F^2}{\big\|\Wb^{(T)}\big\|_2^2} -2K\Bigg|\\
    =&  \Bigg|\frac{2K m^2\big[\log T - \log(m^2\eta^{-1})\big]^2 + O\big(m^2\big[\log T - \log(m^2\eta^{-1})\big]\big)}{m^2\big[\log T - \log(m^2\eta^{-1})\big]^2 + O\big(m^2\big[\log T - \log(m^2\eta^{-1})\big]\big)} -2K\Bigg|\leq O\bigg(\frac{1}{\log T}\bigg).
\end{align*}
This completes the proof for~\eqref{eq:main_result_1}. Next we prove the results for the training loss and test loss. For each input pair $(\Xb, y)$, let $\bnu_1, \cdots, \bnu_s$ represent the basis vectors assigned to its object patches, and  $\bxi_{1}, \cdots, \bxi_{P-s}$ represent the noise vector assigned to its background patches. Then we conclude that $F_{-y}(\Wb_{-y}^{(T)}, \Xb) \leq 1$ if $\la\wb_{-y, r}^{(T)}, \bxi_{k'}\ra \leq 1/P$ because 
\begin{align}\label{eq:output_-F}
    F_{-y}(\Wb_{-y}^{(T)}, \Xb) = \frac{1}{m}\sum_{r=1}^{m} \sum_{k'=1}^{s}\sigma(\la\wb_{-y, r}^{(T)}, \bnu_k\ra)+\frac{1}{m}\sum_{r=1}^{m} \sum_{k'=1}^{P-s}\sigma(\la\wb_{-y, r}^{(T)}, \bxi_{k'}\ra)\leq 1,
\end{align}
where the last inequality holds as $\sigma(\la\wb_{-y, r}^{(T)}, \bnu_k\ra) \leq \sigma(\tilde O(\sigma_0))\leq O(\sigma_0)$ by Lemma~\ref{lemma:opposite_learning} and $\sigma(\la\wb_{-y, r}^{(T)}, \bxi_{k'}\ra) \leq \la\wb_{-y, r}^{(T)}, \bxi_{k'}\ra \leq 1/P$ by our assumptions. On the other hand, $\Xb$ contains at least one basis vector from $\cU_{y}$, and W.L.O.G we denote this basis vector in $\Xb$ as $\bmu_{y, k^*}$. Then we can derive that  
\begin{align}\label{eq:output_F}
    yF_{y}(\Wb_{y}^{(T)}, \Xb) \geq \frac{1}{m}\sigma(\la\wb_{y, r_{y, k^*}}^{(T)}, \bmu_{y, k^*}\ra)\geq \log T - \log (m^2\eta^{-1}) - O(1),
\end{align}
where the last inequality is derived by Lemma~\ref{lemma:phase2_main}. For each $i\in [n], k'\in [P-s_i]$, we have 
\begin{align}\label{eq:upper_bound_bxi}
    \la\wb_{-y, r}^{(T)}, \bxi_{k'}\ra &= \la\wb_{-y, r}^{(0)} +\bXi_{-y, r}^{(T)}, \bxi_{i, k'}\ra \leq \la\wb_{-y, r}^{(0)}, \bxi_{i, k'}\ra+\big\|\bXi_{-y, r}^{(T)}\big\|_2 \big\|\bxi_{i, k'}\big\|_2\notag \\
    &\leq \tilde O(\sigma_0\sigma_{\mathrm{noise}}\sqrt{d}) + O(\sigma_0\sigma_{\mathrm{noise}}\sqrt{nd}) \leq \frac{1}{P},
\end{align}
where the equality holds by the orthogonality between the noise vectors and basis vectors. The first inequality holds by Cauchy-Schwarz inequality. The second inequality is derived by applying the concentration results in Lemmas~\ref{lemma:data_innerproducts}, \ref{lemma:initialization_norms} and $\|\bXi_{-y, r}^{(T)}\|_2\leq O(\sigma_0\sqrt{n})$ in Theorem~\ref{thm:main_result}. The last inequality is by our Condition~\ref{condition:d_sigma0_eta}. Therefore,~\eqref{eq:output_-F} and~\eqref{eq:output_F} hold for all $i\in[n]$, and we can obtain that 
\begin{align*}
    \ell\big(y_i F(\Wb^{(T)}, \Xb_i)\big) &\leq \exp\big(F_{-y_i}(\Wb_{-y_i}^{(T)}, \Xb_i) - F_{y_i}(\Wb_{y_i}^{(T)}, \Xb_i)\big) \notag\\
    &\leq \exp\big(O(1) + \log (m^2\eta^{-1})-\log T\big)\leq  O\bigg(\frac{m^2}{\eta T}\bigg),
\end{align*}
which finishes the proof for~\eqref{eq:main_result_2}. Additionally, if $s_i=1$ with $\bmu_{y_i, k^*}$ being the basis vector contained in the this data point, it holds that
\begin{align*}
    y_iF_{y_i}(\Wb_{y_i}^{(T)}, \Xb_i) 
    &= \frac{1}{m}\sigma(\la\wb_{y_i, r_{y_i, k^*}}^{(T)}, \bmu_{y_i, k^*}\ra) + \frac{1}{m}\sum_{r\neq r_{y_i, k^*}}\sigma(\la\wb_{y_i, r}^{(T)}, \bmu_{y_i, k^*}\ra) + \frac{1}{m}\sum_{r=1}^m\sum_{k'=1}^{P-1}\sigma(\la\wb_{y_i, r}^{(T)}, \bxi_{i, k'}\ra)\\
    &\leq \log T - \log (m^2\eta^{-1}) + O(1),
\end{align*}
where the last inequality holds by Lemma~\ref{lemma:phase2_main} and a similar calculation in~\eqref{eq:upper_bound_bxi}. Then we can obtain that 
\begin{align*}
    \ell\big(y_i F(\Wb^{(T)}, \Xb_i)\big) &\geq \frac{1}{2}\exp\big(-F_{y_i}(\Wb_{y_i}^{(T)}, \Xb_i)\big) \notag\\
    &\geq \frac{1}{2}\exp\big(O(1) + \log (m^2\eta^{-1})-\log T\big)\geq  \Omega\bigg(\frac{m^2}{\eta T}\bigg),
\end{align*}
which finishes the proof that $L_{S}(\Wb^{(T)}) = \Theta\big(\frac{m^2}{\eta T}\big)$ if $s_i=1$ for all $i\in [n]$. Finally, we demonstrate the result for test loss. For a new pair $(\Xb^*, y^*)\sim \cD$ independent of $\{(\Xb_i, y_i)\}_{i=1}^n$, we denote $\cE_{T}$ the event that $\la\wb_{y^*,r}^{(0)} + \bXi_{y^*, r}^{(T)}, \bxi_{k'}^*\ra \leq 1/P$ for all $r\in[m]$ and $k'\in [P-s]$. By the independence among the new pair and training set, we further have $\la\wb_{y^*,r}^{(0)} + \bXi_{y^*, r}^{(T)}, \bxi_{k'}^*\ra$ is a Gaussian random variable with mean $0$ and variance $\sigma_{\mathrm{noise}}^2\|\Pb_{\cN}\cdot\wb_{y^*,r}^{(0)} + \bXi_{y^*, r}^{(T)}\|_2^2\leq \sigma_{\mathrm{noise}}^2O(\|\wb_{y^*,r}^{(0)}\|_2^2\vee \|\bXi_{y^*, r}^{(T)}\|_2^2)\leq \tilde O(\sigma_0^2 \sigma_{\mathrm{noise}}^2d)$. Therefore, by the tail bounds for Gaussian random variables and union bound, we can obtain that,
\begin{align*}
    \PP(\cE_{T}^c) &\leq \sum_{r=1}^m\sum_{k'=1}^{P-s^*}\PP\big(|\la\wb_{y^*,r}^{(0)} + \bXi_{y^*, r}^{(T)}, \bxi_{k'}^*\ra| \geq \frac{1}{P}\big)  \\
    &\leq\exp\big(\log(2mp)-\tilde\Omega(\sigma_0^{-2} \sigma_{\mathrm{noise}}^{-2}d^{-1})\big) \leq \exp(-\Omega(d)),
\end{align*}
where the last inequality holds by our Condition~\ref{condition:d_sigma0_eta}. In the next, we consider separating the test loss into two parts as
\begin{align*}
    \EE[\ell(y^*\cdot f(\Wb^{(T)}, \Xb^*))] = \underbrace{\EE[\ell(y^*\cdot f(\Wb^{(T)}, \Xb^*))\mathds{1}_{\cE_{T}}]}_{I_1} + \underbrace{\EE[\ell(y^*\cdot f(\Wb^{(T)}, \Xb^*))\mathds{1}_{\cE_{T}^c}]}_{I_2}.
\end{align*}
For the first term $I_1$, by the definition of $\cE_{T}$, we can derive that $\ell(y^*\cdot f(\Wb^{(T)}, \Xb^*))\leq O\big(\frac{m^2}{\eta T}\big)$ if $\mathds{1}_{\cE_{T}}=1$ by following the same procedure for training loss. Besides, it always holds that 
\begin{align}\label{eq:ell_bound_general}
    \ell(y^*\cdot f(\Wb^{(T)}, \Xb^*)) &\leq \log\big(1+\exp(F_{-y^*}(\Wb_{-y^*}^{(T)}, \Xb^*))\big)\leq 1+F_{-y^*}(\Wb_{-y^*}^{(T)}, \Xb^*) \notag\\
    &= 1 + \frac{1}{m}\sum_{r=1}^m\sum_{k'=1}^{P-s^*}\sigma\big(\la\wb_{-y^*, r}^{(T)} , \bxi_{k'}^*\ra\big) + \frac{1}{m}\sum_{r=1}^m\sum_{k=1}^{s^*}\sigma\big(\la\wb_{-y^*, r}^{(T)}, \bnu_k^*\ra\big) \notag\\
    &\leq 2 + \tilde O(\sigma_0\sqrt{d})\sum_{k'=1}^{P-s^*}\|\bxi_{k'}^*\|_2,
\end{align}
where the second inequality holds by the fact $\log(1+\exp(x))\leq 1+x, x>0$. The last inequality holds by Lemma~\ref{lemma:opposite_learning} and Cauchy-Schwarz inequality. 
Based on this result, we can obtain that 
\begin{align*}
    I_2&\leq \sqrt{\EE[\mathds{1}_{\cE_{T}^c}]}\sqrt{\EE[\ell(y^*\cdot f(\Wb^{(T)}, \Xb^*))^2]}\\ &\leq\sqrt{\PP(\cE_{T}^c)}\sqrt{4P + \tilde O(\sigma_0^2 d)\EE\Big[\sum_{k'=1}^{P-s^*}\EE[\|\bxi_{k'}^*\|_2^2|s^*]\Big]}\\
    &\leq \exp\big(-\Omega(d)\big)\cdot \tilde O(1)\leq \exp\big(-\Omega(d)\big).
\end{align*}
where the first inequality is by Cauchy-Schwarz inequality. The second inequality is derived by applying~\eqref{eq:ell_bound_general}. The third inequality is by the fact that $\EE[\|\bxi_{k'}\|_2^2] = \sigma_{\mathrm{noise}}^2d$, and our Condition~\ref{condition:d_sigma0_eta} regarding $\sigma_0$, which also implies the last inequality. This completes the proof for~\eqref{eq:main_result_3}.
\end{proof}

\section{Conclusions and future work}
In this paper, we study the rank of convolutional neural networks (CNNs) trained by gradient descent. We theoretically show that the two-layer CNN will converge to a low-rank structure which aligns with the 'clean' complexity of the data sample which further implies that the neural network exhibits robustness of the rank to noises in data. We make numerical experiments both on real data sets MNIST, CIFAR10 and our synthetic data sets. The experiment results support our theoretical findings. We predict this result can also be extended to deeper and more general neural networks, which is a feasible future research direction. 

Our paper studies the ``rank robustness'' phenomenon for two-layer CNN models, while it will be more important and interesting to explore deeper models. Then, we will need to explore a preciser definition of the rank of the model in the deeper setting as it will also be related to the intermediate embedding. Exploring this direction will be one of our future studies. Besides, it is also interesting to investigate whether a similar phenomenon can be also observed for other architectures, such as vision transformer, and other tasks, such as generative model or self-supervise learning. This can better help us understand the distinct training patterns across different tasks with different architectures.

\appendix
\section{Additional related works}
\noindent\textbf{Implicit bias.} There emerges a line of works studying the concept of 'implicit bias', the inherent property of learning algorithms prioritizing a solution with some specific structures, especially some 'simple' structures \citep{neyshabur2014search, DBLP:journals/jmlr/SoudryHNGS18, pmlr-v99-ji19a, NEURIPS2022_ab3f6bbe, xie2024implicit, zhangimplicit}. For the implicit bias study on neural networks,  \citet{lyu2019gradient, ji2020directional} demonstrated that 
$q$-homogeneous neural networks trained by gradient descent converge in direction to a KKT point of the maximum $\ell_2$-margin problem. \citet{lyu2021gradient} proposed a stronger result base on symmetric data assumption and \citet{wang2021implicit} extend the results to adaptive methods. \citet{ji2019gradient, ji2020directional} showed the each layer of deep linear neural networks converges to a rank 1 matrix. \citet{li2020towards} establish the equivalence between the gradient flow of depth-2 matrix factorization and a heuristic rank minimization algorithm. \citet{frei2022implicit} showed that on nearly orthogonal data, gradient flow in leaky ReLU networks will achieve a linear boundary, and the stable rank of the neural networks is always bounded by a constant. \citet{kou2024implicit} extends this result to gradient descent on similar data structures. \citet{timor2023implicit} study the rank minimization on non-linear networks and provide several counter-examples. Besides, \citet{vardi2023implicit} provides a literature review of the existing works of implicit bias on deep neural networks.

\noindent\textbf{Benign over-fitting.} \citet{belkin2019reconciling, belkin2020two} demonstrated the  ``double descent'' population risk curve for many models, containing decision tree and Gaussian and random Fourier feature model. \citet{bartlett2020benign} showed that the benign overfitting in linear regression is correlated with the effective rank of the data covariance, and provided a theoretical bound for over-parameterized minimum norm interpolator. \citet{chatterji2020finite} study the benign overfitting in linear classification for a sub-Gaussian mixture model with noise flipping. \citet{wu2020optimal, hastie2022surprises} study the implicit bias under the regime that dimension and sample increase at a fixed ratio. \citet{liang2020multiple, adlam2020neural, meng2024multiple} explored the multiple descent under different settings. Besides, \citet{cao2022benign, frei2023benign, kou2023benign} study the benign overfitting on two-layer neural networks. 

\section{Proof of Proposition~\ref{prop:data_rank}}\label{section:proof_prop}
In this section, we provide a proof for the Proposition~\ref{prop:data_rank}. We first introduce some notations for further illustration. In Definition~\ref{def:data}, $\cS_i \subsetneqq [P]$ determines index of the label-correlated patches in each input $\Xb_i$ with $i\in [n]$. We denote $s_i =|\cS_i|$, the number of elements in $\cS_i$. With this notation, we demonstrate our proof in the following.

\begin{proof}[Proof of Proposition~\ref{prop:data_rank}]
For data matrix $\hat\Xb$, the summation $\sum_{i=1}^n s_i$ represents the total number of object patches in $\hat \Xb$, while $nP- \sum_{i=1}^n s_i$ indicates the total number of noise patches. By the definition of $\cS_i$, we have $n\leq \sum_{i=1}^n s_i \leq n(P-1)$, and consequently $n\leq nP- \sum_{i=1}^n s_i\leq n(P-1)$. For the data matrix $\hat \Xb$, we consider permutating its columns such that the first $\sum_{i=1}^n s_i$ columns contain all object patches, while the remaining $nP- \sum_{i=1}^n s_i$ columns consist of noise patches. Specifically, we denote $\tilde \Xb = [\tilde\Xb_1, \tilde\Xb_2]\in \RR^{d\times nP}$ the permutated data matrix, where $\tilde\Xb_1 \in \RR^{d\times \sum_{i=1}^n s_i}$ contains only object patches, and $\tilde\Xb_2 \in \RR^{d\times (nP-\sum_{i=1}^n s_i)}$ contains noise patches. Since permutation does not alter the singular values of a matrix, the stable rank of $\hat \Xb$ is equal to that of $\tilde \Xb$. (This is because the permutation matrix is orthogonal.) Therefore, we try to derive the conclusion directly by calculating the stable rank of $\tilde \Xb$.

Additionally, among the total $\sum_{i=1}^n s_i$ object patches, denote $K_{1, k}$ the number of patches assigned to $\bmu_{1, k}$, and $K_{-1, k}$ the number of patches assigned to $\bmu_{-1, k}$ for all $k\in [P]$. We further consider permutating $\tilde \Xb_1$ (as permutation does not alter the singular values) into $\tilde\Xb_1^*$ such that the identical object vectors are grouped together and arranged in order. Then based on the notations regarding the number of distinct basis vectors, we can express $\tilde\Xb_1^*$ as follows:
\begin{align*}
    \tilde\Xb_1^* = [\underbrace{\bmu_{1, 1}, \cdots, \bmu_{1, 1}}_{K_{1, 1}}, \cdots, \underbrace{\bmu_{1, K}, \cdots, \bmu_{1, K}}_{K_{1, K}}, \underbrace{\bmu_{-1, 1}, \cdots, \bmu_{-1, 1}}_{K_{-1, 1}}, \cdots, \underbrace{\bmu_{-1, K}, \cdots, \bmu_{-1, K}}_{K_{-1, K}}] \in \RR^{d\times \sum_{i=1}^n s_i}.
\end{align*}
Due to the orthogonality among $\bmu_{1, k}$'s and  $\bmu_{-1, k}$'s, we have $(\tilde\Xb_1^*)^\top\tilde\Xb_1^*$ is a block diagonal matrix, with each diagonal block having entries equal to $1$. Specifically, we can express $(\tilde\Xb_1^*)^\top\tilde\Xb_1^*$ as
\begin{align*}
(\tilde\Xb_1^*)^\top\tilde\Xb_1^* = \begin{bmatrix}
\mathbf{1}_{K_{1, 1}}\mathbf{1}_{K_{1, 1}}&  & &  & &  \\
 & \ddots & & &  & \\
 & & \mathbf{1}_{K_{1, K}}\mathbf{1}_{K_{1, K}}& &  & \\
 & & &  \mathbf{1}_{K_{-1, 1}}\mathbf{1}_{K_{-1, 1}} &  & \\
 &  & & & \ddots &\\
 &  & & & & \mathbf{1}_{K_{-1, K}}\mathbf{1}_{K_{-1, K}}
\end{bmatrix},
\end{align*}
which clearly indicates the spectral decomposition of $(\tilde\Xb_1^*)^\top\tilde\Xb_1^*$. Then by property of all-ones matrix, the $2K$ non-zero singular values of $\tilde \Xb_1^*$ (also $\tilde \Xb_1$) are $\sqrt{K_{1, k}}$'s and $\sqrt{K_{-1, k}}$'s with $k\in [K]$.
Additionally by Lemma~\ref{lemma:numberofdataIII}, we have 
\begin{align*}
\frac{\sum_{i=1}^n s_i}{2K} - O\Big(\sqrt{n\log (1/\delta)}\Big)\leq K_{1, k}, K_{-1, k} \leq \frac{\sum_{i=1}^n s_i}{2K} + O\Big(\sqrt{n\log (1/\delta)}\Big),
\end{align*}
holds with probability at least $1-\delta/2$.
These results characterize the singular values of $\tilde \Xb_1$, and next we focus on the singular values of $\tilde \Xb_2$. By directly applying Lemma~\ref{lemma:singular_value_Gaussian}, with probability at least $1-\delta/2$, it holds that 
\begin{align*}
     &\lambda_{\min}(\tilde \Xb_2) \geq\sigma_{\mathrm{noise}}\bigg(\sqrt{d} - \sqrt{nP} -O\Big(\sqrt{\log(1/\delta)}\Big)\bigg);\\
     &\lambda_{\max}(\tilde \Xb_2) \leq \sigma_{\mathrm{noise}}\bigg(\sqrt{d} + \sqrt{nP} + O\Big(\sqrt{\log(1/\delta)}\Big)\bigg), 
\end{align*}
where $\lambda_{\min}(\tilde \Xb_2)$ indicates the smallest singular value of $\tilde \Xb_2$ and $ \lambda_{\max}(\tilde \Xb_2)$ indicates the largest singular value of $\tilde \Xb_2$.
Additionally, by applying Lemma~\ref{lemma:singular_value_concatenation}, we know that the singular values of $\tilde\Xb$ are union of singular values of $\tilde\Xb_1$ and $\tilde\Xb_2$. Combining all these results, we can finally conclude that
\begin{itemize}[leftmargin = *]
    \item If $\sigma_{\mathrm{noise}}\sqrt{d} \leq O(1)$, then 
    \begin{align*}
        &\sr(\tilde \Xb) = \frac{\sum_{k=1}^K K_{1, k} + \sum_{k=1}^K K_{-1, k} + \sum_{k'=1}^{nP-\sum_{i=1}^n s_i} \lambda_i^2(\tilde\Xb)}{ \max_{j\in \{\pm 1\},k\in[K]}K_{j, k}}\leq 2K + O(\sigma_{\mathrm{noise}}^2d);\\
        &\sr(\tilde \Xb) = \frac{\sum_{k=1}^K K_{1, k} + \sum_{k=1}^K K_{-1, k} + \sum_{k'=1}^{nP-\sum_{i=1}^n s_i} \lambda_i^2(\tilde\Xb)}{ \max_{j\in \{\pm 1\},k\in[K]}K_{j, k}}\geq 2K -  O\bigg(\sqrt{\frac{\log(1/\delta)}{n}}\bigg),
    \end{align*}
    where $\lambda_i(\tilde\Xb)$ denote the $i$-th singular value of $\tilde\Xb$ in descending order. 
    \item If $\Omega(1)\leq \sigma_{\mathrm{noise}}\sqrt{d} \leq O(\sqrt{n})$, then 
    \begin{align*}
        &\sr(\tilde \Xb) = \frac{\sum_{k=1}^K K_{1, k} + \sum_{k=1}^K K_{-1, k} + \sum_{k'=1}^{nP-\sum_{i=1}^n s_i} \lambda_i^2(\tilde\Xb)}{ \max_{j\in \{\pm 1\},k\in[K]}K_{j, k}}\geq 2K +  \Omega(1).
    \end{align*}
    \item If $\sigma_{\mathrm{noise}}\sqrt{d} \geq \Omega(\sqrt{n})$, then 
    \begin{align*}
        &\sr(\tilde \Xb) = \frac{\sum_{k=1}^K K_{1, k} + \sum_{k=1}^K K_{-1, k} + \sum_{k'=1}^{nP-\sum_{i=1}^n s_i} \lambda_i^2(\tilde\Xb)}{\lambda_1^2(\tilde\Xb)}\geq n .
    \end{align*}
\end{itemize}
\end{proof}

\section{Proof in Section~\ref{section:overview_proof}}\label{section:proof_I}
In this section, we provide a detailed proof for the Lemmas in Section~\ref{section:overview_proof}. We leave the proof for Theorem~\ref{thm:main_result} in the next section as it will utilize the several conclusions in this section.
\subsection{Decomposition}
In this subsection, 
we introduce the following decomposition of $\wb_{j,r}^{(t)}$ for our illustration, which considers describing the learning of each noise $\bxi_{i, k'}$ on each filter $\wb_{j, r}$.

\begin{definition}\label{def:w_decomposition}
Let $\wb_{j,r}^{(t)}$ for $j\in \{\pm 1\}$, $r \in [m]$ be the convolution filters of the CNN at the $t$-th iteration of gradient descent. Then there exist unique coefficients $\gamma_{j,k,r}^{(t)} \geq 0$, $\beta_{j, k, r}^{(t)}\leq 0$, and $\rho_{j,r,i,k'}^{(t)}$ such that 
\begin{align}\label{eq:w_decomposition_detail}
    \wb_{j,r}^{(t)} =& \wb_{j,r}^{(0)} + \sum_{k=1}^K  \gamma_{j,k,r}^{(t)} \cdot \bmu_{j, k} + \sum_{k=1}^K \beta_{j,k,r}^{(t)} \cdot \bmu_{-j, k} + \sum_{ i = 1}^n \sum_{k'=1}^{P-s_i} \rho_{j,r,i,k'}^{(t) }\cdot \| \bxi_{i,k'} \|_2^{-2} \cdot \bxi_{i,k'}
\end{align}
We further denote $\bar\rho_{j,r,i,k'}^{(t)} := \rho_{j,r,i,k'}^{(t)}\ind(\rho_{j,r,i,k'}^{(t)} \geq 0)$, $\underline\rho_{j,r,i,k'}^{(t)} := \rho_{j,r,i,k'}^{(t)}\ind(\rho_{j,r,i,k'}^{(t)} \leq 0)$. Then we have,
\begin{align*}
\wb_{j,r}^{(t)} = &\wb_{j,r}^{(0)} + \sum_{k=1}^K j \cdot \gamma_{j,k,r}^{(t)} \cdot \|  \bmu_k + \sum_{k=1}^K \cdot \beta_{j,k,r}^{(t)} \cdot \bmu_{-j, k} \\
&+ \sum_{ i = 1}^n \sum_{k'=1}^{P-s_i} \bar\rho_{j,r,i,k'}^{(t) }\cdot \| \bxi_{i,k'} \|_2^{-2} \cdot \bxi_{i,k'}
+ \sum_{ i = 1}^n \sum_{k'=1}^{P-s_i} \underline\rho_{j,r,i,k'}^{(t) }\cdot \| \bxi_{i,k'} \|_2^{-2} \cdot \bxi_{i,k'}
\end{align*}
\end{definition}
Then, instead of directly analyzing $\bXi_{j, k}^{(t)}$, we prove some result for $\rho_{j, r, i, k'}$ and extend the results of $\rho_{j, r, i, k'}$ to $\bXi_{j, k}^{(t)}$. Besides, we define two set notations: $\cI_{1, k}$ is the set of data points containing basis vector $\bmu_{1, k}$ in their object patches, i.e., $\cI_{1,k} = \big\{i|\bmu_{1, k} \in\{\nu_{i, 1}, \cdots, \nu_{i, s_i}\},\ \text{and} \ i\in [n]\big\}$, and similarly $\cI_{-1, k}$ is the set of data points containing basis vector $\bmu_{-1, k}$ in their object patches. Additionally, $\cJ_{1, k}$ is the set of data points containing only one basis vector $\bmu_{1, k}$ in their object patches, i.e. $\cJ_{1, k} = \big\{i|s_i=1, \ \nu_{i, 1} =\bmu_{1, k}, \ \text{and} \ i\in [n]\big\}$
\subsection{Proof of Lemma~\ref{lemma:opposite_learning}}

We first provide proof for Lemma~\ref{lemma:opposite_learning}, as it will be applied in the following proof.

\begin{proof}[Proof of Lemma~\ref{lemma:opposite_learning}]
We only consider the case for $j=1$, and the proof for the case $j=-1$ is totally identical. By multiplying $\bmu_{-1, k}$ on both sides of~\eqref{eq:gdupdate3}, and the orthogonality between basis and noises, we obtain that
\begin{align*}
    \la\wb_{1, r}^{(t+1)},\bmu_{-1, k}\ra &=  \la\wb_{1, r}^{(t)},\bmu_{-1, k}\ra + \frac{\eta}{nm}\sigma'\big(\la\wb_{1, r}^{(t)},\bmu_{-1, k}\ra\big)\sum_{i\in \cI_{-1, k}} \ell_i'^{(t)}.
\end{align*}
If $\la\wb_{1, r}^{(0)},\bmu_{-1, k}\ra < 0$, then $\sigma'\big(\la\wb_{1, r}^{(t)},\bmu_{-1, k}\ra\big) = 0$, which implies that $\la\wb_{1, r}^{(t)},\bmu_{-1, k}\ra = \la\wb_{1, r}^{(0)},\bmu_{-1, k}\ra$ for all $t> 0$. Additionally by Lemma~\ref{lemma:initialization_norms}, we have $\big|\la\wb_{1, r}^{(t)},\bmu_{-1, k}\ra\big| = \big|\la\wb_{1, r}^{(0)},\bmu_{-1, k}\ra\big| \leq \tilde O(\sigma_0\|\bmu\|_2)$. If $\la\wb_{1, r}^{(0)},\bmu_{-1, k}\ra \geq 0$, then $\la\wb_{1, r}^{(t)},\bmu_{-1, k}\ra$ is non-increasing until it first reaches a negative value, after which it remains unchanged. Therefore, we have 
\begin{align*}
    \big|\la\wb_{1, r}^{(t)},\bmu_{-1, k}\ra\big| &\leq \Bigg|\frac{\eta}{nm}\sigma'\big(\la\wb_{1, r}^{(0)},\bmu_{-1, k}\ra\big)\sum_{i\in \cI_{-1, k}} \ell_i'^{(t)}\Bigg|\vee\big|\la\wb_{1, r}^{(0)},\bmu_{-1, k}\ra\big| \\
    &\leq \Bigg|\frac{\eta}{m}\big(\la\wb_{1, r}^{(0)},\bmu_{-1, k}\ra\big)^{q-1}\Bigg|\vee\big|\la\wb_{1, r}^{(0)},\bmu_{-1, k}\ra\big| \leq \tilde O(\sigma_0\|\bmu\|_2),
\end{align*}
where the last inequality holds by Lemma~\ref{lemma:initialization_norms} and Condition~\ref{condition:d_sigma0_eta}. This completes the proof.
\end{proof}

\subsection{Preliminary lemmas}
Before we prove the Lemma~\ref{lemma:phase1_main}, we first present and prove several lemmas that will be used for the proof of Lemma~\ref{lemma:phase1_main}. We define $r_{j, k, t} =\argmax_{r\in [m]} \la\wb_{j,r}^{(t)}, \bmu_{j, k}\ra$. Then we have the following lemma demonstrating that the filter with the largest inner product with some basis vector at initialization will always have the largest inner product with this basis vector during the whole training process.
\begin{lemma}\label{lemma:initial_large}
    Under Condition~\ref{condition:d_sigma0_eta}, we have $r_{j, k, t}  = r_{j, k, 0}$ for all $j \in \{+1, -1\}, k \in [K]$ and $t>0$. Moreover, if $\{j, k\}\neq \{j', k'\}$, then $r_{j, k, 0}\neq r_{j', k', 0}$ holds with probability at least $1-O(1/m)$.
\end{lemma}
\begin{proof}[Proof of Lemma~\ref{lemma:initial_large}]
We first prove that for all $r, r' \in [m]$,  if $\la\wb_{j,r}^{(0)}, \bmu_{j, k}\ra \geq \la \wb_{j,r'}^{(0)}, \bmu_{j, k}\ra$, then it holds that $\la\wb_{j,r}^{(t)}, \bmu_{j, k}\ra \geq \la \wb_{j,r'}^{(t)}, \bmu_{j, k}\ra$ for all $t$. By multiplying $\bmu_{j, k}$ on both sides of~\eqref{eq:gdupdate3}, and the orthogonality between basis vectors and noises, we obtain that
\begin{align}\label{eq:wmu_update}
    \la\wb_{j, r}^{(t+1)},\bmu_{j, k}\ra &=  \la\wb_{j, r}^{(t)},\bmu_{j, k}\ra -\frac{\eta}{nm}\sum_{i\in \cI_{j, k}} \ell'[ y_i \cdot f(\Wb^{(t)},\Xb_i) ]\sigma'\big(\la\wb_{j, r}^{(t)},\bmu_{j,k}\ra\big) j y_i \notag\\
    &= \la\wb_{j, r}^{(t)},\bmu_{j, k}\ra -\frac{\eta}{nm}\sigma'\big(\la\wb_{j, r}^{(t)},\bmu_{j, k}\ra\big)\sum_{i\in \cI_{j, k}} \ell_i'^{(t)}j y_i. 
\end{align}
By definition of $\cI_{j, k}$, it is clear that if $i \in \cI_{j, k}$, then $y_i = j$. Therefore $\la\wb_{j,r}^{(t)}, \bmu_{j,k}\ra$ is always non-decreasing. Moreover, if $\la\wb_{j,r}^{(0)}, \bmu_{j, k}\ra <0$, then $\sigma'(\la\wb_{j,r}^{(0)}, \bmu_{j, k}\ra)=0$. This implies that $\la\wb_{j,r}^{(t)}, \bmu_{j, k}\ra =\la\wb_{j,r}^{(0)}, \bmu_{j, k}\ra <0$ holds for all $t > 0$. Therefore, for $r, r'\in [m]$, it is straightforward that $\la\wb_{j,r}^{(t)}, \bmu_{j, k}\ra \geq \la \wb_{j,r'}^{(t)}, \bmu_{j, k}\ra$ for all $t>0$ if $\la\wb_{j,r}^{(0)}, \bmu_{j, k}\ra \geq 0\geq \la \wb_{j,r'}^{(0)}, \bmu_{j, k}\ra$. In the following, we consider the case where $\la\wb_{j,r}^{(0)}, \bmu_{j, k}\ra \geq  \la \wb_{j,r'}^{(0)}, \bmu_{j, k}\ra\geq 0$. We first simply \eqref{eq:wmu_update} as
\begin{align}\label{eq:wmu_update_positive}
    \la\wb_{j, r}^{(t+1)},\bmu_{j,k}\ra 
    &= \la\wb_{j, r}^{(t)},\bmu_{j, k}\ra -\frac{\eta}{nm}\sigma'\big(\la\wb_{j, r}^{(t)},\bmu_{j, k}\ra\big)\sum_{i\in \cI_{j, k}} \ell_i'^{(t)}.
\end{align}
And we could notice that the only item specific to filter $r$ in formula~\eqref{eq:wmu_update_positive} is the inner product $\la\wb_{j, r}^{(t)}, \bmu_{j,k}\ra$. In another word, if we let $x^{(t)}_r = \la\wb_{j, r}^{(t)}, \bmu_{j,k}\ra$, then the recursion ~\eqref{eq:wmu_update_positive} of the positive sequences $\{x^{(t)}_r\}_{t=0}^\infty$ could be simplified as,
\begin{align*}
    x^{(t+1)}_r
    = x^{(t)}_r + \eta C_t \sigma'\big(x^{(t)}_r\big)
\end{align*}
where $C_t = \frac{1}{nm}\sum_{i\in \cI_{j,k}} \ell_i'^{(t)}$ is independent of filter index $r$, and $\sigma'(\cdot)$ is a non-decreasing function. Therefore we conclude that a filter with a larger initialization will always have a larger increment in each iteration, which completes the proof for the case that $\la\wb_{j,r}^{(0)}, \bmu_{j, k}\ra \geq  \la \wb_{j,r'}^{(0)}, \bmu_{j, k}\ra\geq 0$. Combined these results, we can conclude that if $r\neq r_{j, k, 0}$, then $\la\wb_{j, r}^{(t)},\bmu_{j,k}\ra\leq \la\wb_{j, r_{j, k, 0}}^{(t)},\bmu_{j,k}\ra$ for all $t$. Since the initialization of $\wb_{j, r}^{(0)}$ is i.i.d. Gaussian random vectors, we conclude that $P(r_{j, k, 0}= r_{j', k', 0})=\frac{1}{m}$ for different pair of $\{j, k\}$.
\end{proof}
Since the filter with the largest inner product with any specific basis vector is consistent during the whole training process, in the following paragraphs, we can denote $r_{j, k, t}$ by $r_{j, k}$ for simplicity and $r_{j, k}$'s are distinct for different pair of $\{j, k\}$. 

Next, we introduce and prove the following Lemma~\ref{lemma:F-yi_li'bd} and Lemma~\ref{lemma:relation_Xi_xi} which will be helpful.  Lemma~\ref{lemma:F-yi_li'bd} characterize the relationship between $-\ell_i'$ and the output of $F_{y_i}$ when $|\rho_{j,r,i, k'}^{(t)}|$ is small. Lemma~\ref{lemma:relation_Xi_xi} show that when $|\rho_{j,r,i, k'}^{(t)}|$ is small $\|\bXi_{j, r}^{(t)}\|_2$ is also small.

\begin{lemma}\label{lemma:F-yi_li'bd}
    Suppose that Condition~\ref{condition:d_sigma0_eta} holds and $|\rho_{j,r,i, k'}^{(t)}| \leq O(\sigma_0 \sigma_{\mathrm{noise}} \sqrt{d})$ for all $j\in \{\pm 1\}, r\in[m]$, $i \in [n]$ and $k' \in [P - s_i]$, then we have
    \begin{align*}
        &\la\wb_{j, r}^{(t)}, \bxi_{i,k'}\ra \leq \tilde O(\sigma_0 \sigma_{\mathrm{noise}} \sqrt{d});\\
        &F_{-y_i}(\Wb_{-y_i}^{(t)}, \Xb_i) \leq 1;\\
        &|\ell_i'^{(t)}| =\Theta\rbr{e^{-F_{y_i}(\Wb_{y_i}^{(t)}, \Xb_i)}}.
    \end{align*}
\end{lemma}
\begin{proof}[Proof of Lemma~\ref{lemma:F-yi_li'bd}]
The iterative rule for $\la\wb_{j, r}^{(t)}, \bxi_{i,k'}\ra$ can be derived by multiple by $\bxi_{i,k'}$ on both sides of ~\eqref{eq:gdupdate}, then we obtain that
\begin{align*}
    \la\wb_{j, r}^{(t)}, \bxi_{i,k'}\ra &= \la\wb_{j, r}^{(0)}, \bxi_{i,k'}\ra + \rho_{j, r, i, k'}^{(t)} +  \sum_{i'\neq i}\sum_{k''=1}^{P-s_i} \rho_{j, r, i', k''}^{(t)}\frac{\la\bxi_{i, k'},\bxi_{i', k''}\ra}{\|\bxi_{i', k''}\|_2^2} + \sum_{k''\neq k'}\rho_{j, r, i, k''}^{(t)}\frac{\la\bxi_{i, k'},\bxi_{i, k''}\ra}{\|\bxi_{i, k''}\|_2^2} \notag\\
        &\leq \tilde O(\sigma_0 \sigma_{\mathrm{noise}} \sqrt{d}).
\end{align*}
The last inequality holds because the first term $\la\wb_{-y_i, r}^{(0)}, \bxi_{i,k'}\ra = \tilde O(\sigma_0 \sigma_{\mathrm{noise}} \sqrt{d})$ by applying Lemma~\ref{lemma:initialization_norms}, the second term $\rho_{j, r, i, k'}^{(t)} = O(\sigma_0 \sigma_{\mathrm{noise}} \sqrt{d})$ by our assumption. 
For the third term and the forth term, $\rho_{j, r, i', k''}^{(t)}, \rho_{j, r, i, k''}^{(t)} = \tilde O(\sigma_0 \sigma_{\mathrm{noise}} \sqrt{d})$ by our assumption and $\frac{\la\bxi_{i, k'},\bxi_{i', k''}\ra}{\|\bxi_{i', k''}\|_2^2}, \frac{\la\bxi_{i, k'},\bxi_{i, k''}\ra}{\|\bxi_{i, k''}\|_2^2} \leq \tilde O(1/\sqrt d)$ by Lemma~\ref{lemma:data_innerproducts}. Then based on our Condition~\ref{condition:d_sigma0_eta} about $n$ and $d$, the last two terms are also $\tilde O(\sigma_0 \sigma_{\mathrm{noise}} \sqrt{d})$. Now for $F_{-y_i}(\Wb_{-y_i}^{(t)}, \Xb_i)$, its value is determined by $\la\wb_{-y_i, r}^{(t)}, \bmu_{y_i, k}\ra$ and $\la\wb_{-y_i, r}^{(t)}, \bxi_{i,k'}\ra$. Lemma~\ref{lemma:opposite_learning} implies that $\la\wb_{-y_i, r}^{(t)}, \bmu_{y_i, k}\ra \leq \tilde O(\sigma_0 )$.
   Combining with the previous result about $\la\wb_{j,r}^{(t)}, \xi_{i, k'}\ra$, we can present a bound of $F_{-y_i}(\Wb_{-y_i}^{(t)}, \Xb_i)$ as
    \begin{align*}
        F_{-y_i}(\Wb_{-y_i}^{(t)}, \Xb_i) &\leq \frac{1}{m}\sum_{r=1}^m\sum_{k=1}^K \sigma(\la\wb_{-y_i, r}^{(t)}, \bmu_{y_i, k}\ra) + \frac{1}{m}\sum_{r=1}^m\sum_{k=1'}^{P-s_i} \sigma(\la\wb_{-y_i, r}^{(t)}, \bxi_{i,k'}\ra)\\
        &\leq \frac{2P}{q \kappa^{q-1}}\cdot\max\{\tilde O(\sigma_0 ), \tilde O(\sigma_0 \sigma_{\mathrm{noise}} \sqrt{d})\}^q \leq 1.
    \end{align*}
    The last inequality is from our Condition~\ref{condition:d_sigma0_eta} about $\sigma_0$. Finally by the definition of $\ell'(\cdot)$, it is clear that,
    \begin{align*}
        |\ell_i'^{(t)}| = \frac{1}{1+e^{y_i[F_{+1}(\Wb_{+1}^{(t)}, \Xb_i)- F_{-1}(\Wb_{-1}^{(t)}, \Xb_i)]}}= \frac{1}{1+e^{F_{y_i}(\Wb_{y_i}^{(t)}, \Xb_i) -F_{-y_i}(\Wb_{-y_i}^{(t)}, \Xb_i)}}.
    \end{align*}
    By the fact $F_{+1}(\cdot), F_{-1}(\cdot) \geq 0$, the lower bound is straightforward that
    \begin{align*}
        |\ell_i'^{(t)}| = \frac{1}{1+e^{F_{y_i}(\Wb_{y_i}^{(t)}, \Xb_i) -F_{-y_i}(\Wb_{-y_i}^{(t)}, \Xb_i)}} \geq \frac{1}{2e^{F_{y_i}(\Wb_{y_i}^{(t)}, \Xb_i)}}
    \end{align*}
    On the other side, since $F_{-y_i}(\Wb_{-y_i}^{(t)}, \Xb_i)\leq 1$, we obtain that
    \begin{align*}
        |\ell_i'^{(t)}| = \frac{1}{1+e^{F_{y_i}(\Wb_{y_i}^{(t)}, \Xb_i) -F_{-y_i}(\Wb_{-y_i}^{(t)}, \Xb_i)}} \leq e\cdot e^{-F_{y_i}(\Wb_{y_i}^{(t)}, \Xb_i)}
    \end{align*}
    The upper and lower bound of $|\ell_i'^{(t)}|$ indicates that $|\ell_i'^{(t)}| = \Theta\Big(e^{-F_{y_i}(\Wb_{y_i}^{(t)}, \Xb_i)}\Big)$.
\end{proof}

\begin{lemma}\label{lemma:relation_Xi_xi}
Suppose that Condition~\ref{condition:d_sigma0_eta} holds. If $|\rho_{j, r, i, k'}^{(t)}| \leq a$, then we have $\|\bXi_{j, r}^{(t)}\|_2^2 \leq 2nPa^2\sigma_{\mathrm{noise}}^{-2}d^{-1} $ for all $j\in \{-1, +1\}$, $r\in[m]$, $i\in[n]$ and $k'\in [P-s_i]$, . Here $a$ could be any positive number.
\end{lemma}

\begin{proof}[Proof of Lemma~\ref{lemma:relation_Xi_xi}]
By deposition defined in Definition~\ref{def:w_decomposition}, we have \begin{align*}
    \|\bXi_{j, r}^{(t)}\|_2^2 &= \sum_{i=1}^{n}\sum_{k'=1}^{P-s_i} [\rho_{j,r,i,k'}^{(t)}]^2\|\bxi_{i, k'}\|_2^{-2} + \sum_{\{i, k\}\neq\{i', k'\}} \rho_{j,r,i,k}^{(t)}\rho_{j,r,i',k'}^{(t)}\frac{\la\bxi_{i, k},\bxi_{i', k'}\ra}{\|\bxi_{i, k}\|_2^2\|\bxi_{i', k'}\|_2^2}\\
    &\leq 2n(P-1)a^2\sigma_{\mathrm{noise}}^{-2}d^{-1} +  a^2\sigma_{\mathrm{noise}}^{-2}d^{-1} \leq 2nPa^2\sigma_{\mathrm{noise}}^{-2}d^{-1}, 
\end{align*}
where the first inequality is from Lemma~\ref{lemma:data_innerproducts} and our Condition~\ref{condition:d_sigma0_eta}.
\end{proof}

\subsection{Proof of Lemma~\ref{lemma:phase1_main}}
Now, we are ready to prove Lemma~\ref{lemma:phase1_main}
\begin{proof}[Proof of Lemma~\ref{lemma:phase1_main}]
By Lemma~\ref{lemma:relation_Xi_xi}, to show $\|\bXi_{j, r}^{(t)}\|_2^2 \leq \sigma_0^2 n P/ 2$, it suffices to show that $\max_{j, r, i, k'}|\rho_{j, r, i, k'}^{(t)}| \leq \sigma_0\sigma_{\mathrm{noise}}\sqrt{d}/2$. We will show it by induction, and we assume it holds when we prove the first result. For each basis vector $\bmu_{j, k}\in \cU_{j}$ with $j\in\{-1,+1\}$ and $k\in [K]$, we define an idealized filter $\tilde \wb_{j, k}$ satisfying that $\big(1+\Theta(\frac{\delta}{m^2})\big)\la\tilde \wb_{j, k}^{(0)}, \bmu_{j, k}\ra = \la\wb_{j, r_{j, k, 0}}^{(0)},\bmu_{j, k}\ra$. Besides, we assume that $\la\tilde \wb_{j, k}^{(t)}, \bmu_{j, k}\ra$ also follows the iterative rule in \eqref{eq:wmu_update_positive}. The reason for such a definition is that if $r\neq r_{j, k}$, we have $\la \wb_{j, r}^{(0)}, \bmu_{j,k}\ra \leq \la \tilde\wb_{j, k}^{(0)}, \bmu_{j,k}\ra$ by Lemma~\ref{lemma:initialization_norms}, which further implies that $\la \wb_{j, r}^{(t)}, \bmu_{j,k}\ra \leq \la \tilde\wb_{j, k}^{(t)}, \bmu_{j,k}\ra$ by Lemma~\ref{lemma:initial_large} for all $t > 0$.
For filters $\wb_{j, r_{j,k}}$ and the idealized filter $\tilde \wb_{j,k}$ defined above, $\la\wb_{j, r_{j,k}}^{(t)},j\bmu_k\ra$ and $\la\tilde\wb_{j, k}^{(t)},j\bmu_k\ra$ follows the same iterative rule in \eqref{eq:wmu_update_positive}. Let's rewrite the iterative formula here for better readability for the readers:
\begin{align*}
     \la\wb_{j, r}^{(t+1)},\bmu_{j, k}\ra 
    &= \la\wb_{j, r}^{(t)},\bmu_{j, k}\ra -\frac{\eta}{nm}\sigma'\big(\la\wb_{j, r}^{(t)},\bmu_{j, k}\ra\big)\sum_{i\in \cI_{j, k}} \ell_i'^{(t)} \\
    &= \la\wb_{j, r}^{(t)},\bmu_{j, k}\ra + \eta C_t \sigma'\big(\la\wb_{j, r}^{(t)},\bmu_{j, k}\ra\big) ,
\end{align*}
where $C_t = \frac{1}{nm}\sum_{i\in \cI_{j, k}} \ell_i'^{(t)}\leq \frac{1}{m}$.
Define $T_{1, j, k}$ be the first time such that $\la\wb_{j, r_{j,k}}^{(t)},\bmu_{j, k}\ra \geq \kappa$ and $T'$ be the first time such that $\la\tilde\wb_{j, k}^{(t)},\bmu_{j, k}\ra \geq \frac{1}{4K m}$. Since we have $\frac{\la\wb_{j, r_{j,k}}^{(0)},\bmu_{j, k}\ra}{\la\tilde\wb_{j, k}^{(0)},\bmu_{j, k}\ra} = 1+\Theta(\frac{\delta}{m^2})$ by definition of the idealized filter $\tilde \wb_{j,k}$. By checking the conditions in Lemma~\ref{lemma:sequence_compare}, we can conclude $T_{1, j, k} < T'$, which implies that $\la\tilde\wb_{j, k}^{(T_{1, j, k})},\bmu_{j, k}\ra < \frac{1}{4Km}$. As we have demonstrated above,  $\la\wb_{j,r}^{(t)}, \bmu_{j, k}\ra \leq \la\tilde\wb_{j, k}^{(t)},\bmu_{j, k}\ra$ for all $t>0$ and $r\neq r_{j, k}$. Therefore, we can obtain that $\la\wb_{j,r}^{(T_{1, j, k})}, \bmu_{j, k}\ra \leq \frac{1}{4Km}$ for all $r\neq r_{j, k}$. On the other hand, $\la\wb_{j,r}^{(t)}, \bmu_{j, k}\ra <0$ can only occur if  $\la\wb_{j,r}^{(0)}, \bmu_{j, k}\ra <0$, which further implies that $\la\wb_{j,r}^{(t)}, \bmu_{j, k}\ra = \la\wb_{j,r}^{(0)}, \bmu_{j, k}\ra$ for all $t>0$. Combined with the initialization concentration results and Condition~\ref{condition:d_sigma0_eta} that $\la\wb_{j,r}^{(0)}, \bmu_{j, k}\ra\leq \tilde O(\sigma_0)\leq O(1/m)$, we finally obtain that $\big|\la\wb_{j,r}^{(T_{1, j, k})}, \bmu_{j, k}\ra \big|\leq \frac{1}{4Km}$. Next, we try to derive the bound for $T_{1, j, k}$. As we assume $|\rho_{j,r,i,k'}| \leq  \sigma_0\sigma_{\mathrm{noise}}\sqrt{d}/2$, then for all $t \leq T_{1, j, k}$ and $i \in \cJ_{j, k}$, we apply Lemma~\ref{lemma:F-yi_li'bd} and derive that
\begin{align*}
    -\ell_{i}^{(t)} &\geq \frac{1}{2e}\exp{\big(-F_{j}(\Wb_{j}^{(t)}, \Xb_i)\big)}\\
    &\geq \frac{1}{2e}\exp{\bigg(-\frac{1}{m}\sum_{r=1}^m\Big[\sigma\big(\la\wb_{j, r}^{(t)},\bmu_{j,k}\ra\big)+\sum_{k'=1}^{P-1}\sigma\big(\la\wb_{j, r}^{(t)},\bxi_{i, k'}\ra\big)\Big]\bigg)}\geq \frac{1}{2e^2}.
\end{align*}
This is because
\begin{align*}
    \frac{1}{m}\sum_{r=1}^m\Big[\sigma\big(\la\wb_{j, r}^{(t)},\bmu_{j,k}\ra\big)+\sum_{k'=1}^{P-1}\sigma\big(\la\wb_{j, r}^{(t)},\bxi_{i, k'}\ra\big)\Big] \leq 1
\end{align*}
by $\la\wb_{j, r_{j, k}}^{(t)}, \bmu_{j, k}\ra \leq \kappa$, $\la\wb_{j, r}^{(t)}, \bmu_{j, k}\ra\leq \frac{1}{4Km}$ for $r\neq r_{j, k}$, and $\la\wb_{j, r}^{(t)}, \xi_{i,k'}\ra \leq \tilde O(\sigma_0 \sigma_{\mathrm{noise}} \sqrt{d})$ for all $r\in[m]$ and $k'\in[P-1]$. Therefore, we get a lower bound for $C_t$ as
\begin{align*}
    C_t = \frac{1}{nm}\sum_{i\in \cI_{j, k}} \ell_i'^{(t)} \geq \frac{1}{2e^2nm}|\cJ_{j, k}|\geq \frac{\pi}{16e^2Km}.
\end{align*}
The last inequality is by Lemma~\ref{lemma:numberofdataII}, and $\pi$ is a positive constant solely depending on $K$, which implies that $\pi = \Theta(1)$. And it is straight forward that  $C_t \leq \frac{\|\bmu\|_2^2}{m}$. Then by Lemma~\ref{lemma:sequence_time} and Lemma~\ref{lemma:initialization_norms}, we can obtain that 
\begin{align*}
    &T_{1, j, k} \leq \frac{32e^2Km}{\pi\eta (\la\wb_{j, r_{j, k}}, \bmu_{j, k}\ra)^{q-2}} \leq  \frac{2^{q+3}e^2Km}{\pi\eta \sigma_0^{q-2} };\\
    &T_{1, j, k} \geq \frac{m}{2\eta \|\bmu\|_2^2(\la\wb_{j, r_{j, k}}, \bmu_{j, k}\ra)^{q-2}} \geq  \frac{m}{4
    \eta \sigma_0^{q-2} \log m}.
\end{align*}
Since both the lower bound and upper bound of $T_{1,j, k}$ is independent of $j$ and $k$, we conclude we can find a time $T_1 =  \tilde\Theta\big(\frac{m}{\eta \sigma_0^{q-2} }\big)$ such that the preceding results hold at $T_1$ for all $j\in\{-1, +1\}$ and $k\in [K]$. Finally we use induction to prove that $\max_{j, r, i, k'}|\rho_{j, r, i, k'}^{(t)}| \leq \sigma_0\sigma_{\mathrm{noise}}\sqrt{d}/2$. For simplicity we denote $\phi^{(t)} = \max_{j, r, i, k'}|\rho_{j, r, i, k'}^{(t)}|$. Obviously $\phi^{(0)} =  0$, and we suppose that exists $\tilde T \leq T_1$ such that $\phi^{(t)} \leq \sigma_0\sigma_{\mathrm{noise}}\sqrt{d}/2$ holds for all $0<t<\tilde T-1$. Then by the iterative rule for $\rho_{j, r, i, k'}^{(t)}$, we have
\begin{align*}
    \phi^{(t+1)}&\leq \phi^{(t)} + \max_{j, r, i, k'}\frac{\eta\|\bxi_{i, k'}\|_2^2}{\kappa^{q-1} n m}\Bigg|\la\wb_{j, r}^{(0)}, \bxi_{i,k'}\ra+\phi^{(t)}\bigg(1+ \sum_{i'\neq i}\sum_{k''=1}^{P-s_i}\frac{\la\bxi_{i, k'},\bxi_{i', k''}\ra}{\|\bxi_{i', k''}\|_2^2} + \sum_{k''\neq k'}\frac{\la\bxi_{i, k'},\bxi_{i, k''}\ra}{\|\bxi_{i, k''}\|_2^2}\bigg)\Bigg|^{q-1}\\
    &\leq \phi^{(t)} + \frac{[6\log(m\vee n)]^{(q+1)/2}\eta\sigma_0^{q-1}\sigma_{\mathrm{noise}}^{q+1}d^{(q+1)/2}}{\kappa^{q-1}nm}.
\end{align*}
By taking the telescoping sum, we have $\phi^{(\tilde T)}\leq T_1 \cdot\frac{[6\log(m\vee n)]^{(q+1)/2}\eta\sigma_0^{q-1}\sigma_{\mathrm{noise}}^{q+1}d^{(q+1)/2}}{\kappa^{q-1}nm}\leq \sigma_0\sigma_{\mathrm{noise}}\sqrt{d}/2$ by the formula for $T_1\leq \frac{2^{q+3}e^2Km}{\pi\eta \sigma_0^{q-2}}$ and the assumption that $\frac{\sigma_{\mathrm{noise}}^qd^{q/2}}{n}\leq \frac{\pi\kappa^{q-1}}{2^{q+4}e^2K[6\log(m\vee n)]^{(q+1)/2}}\leq \tilde O(1)$. Since then, we have finished all the proof for Lemma~\ref{lemma:phase1_main}.
\end{proof}

\subsection{Proof of Lemma~\ref{lemma:phase2_main}}
During the phase I, we always threat $-\ell_i' = \Theta(1)$, while in this phase, as the increasing of $\la\wb_{j, r_{j, k}}^{(t)} , \bmu_{j, k}\ra$, we can not regard $-\ell_i' = \Theta(1)$ since the training loss will eventually converge.
\begin{proof}[Proof of Lemma~\ref{lemma:phase1_main}]
By Lemma~\ref{lemma:relation_Xi_xi}, to show $\|\bXi_{j, r}^{(t)}\|_2^2 \leq 2\sigma_0^2 n P$, it suffices to show that $\max_{j, r, i, k'}|\rho_{j, r, i, k'}^{(t)}| \leq \sigma_0\sigma_{\mathrm{noise}}\sqrt{d}$.
Similar to the proof of Phase I,
we first prove the result for $\la\wb_{j, r_{j, k}}^{(t)} , \bmu_{j, k}\ra$ and then use induction to prove the result for $\max_{j, r, i, k'}|\rho_{j, r, i, k'}^{(t)}|$ and $\la\wb_{j, r}^{(t)} , j\bmu_k\ra$ with $r\neq r_{j,k}$. We assume the results for $\max_{j, r, i, k'}|\rho_{j, r, i, k'}^{(t)}|$ and $\la\wb_{j, r}^{(t)} , \bmu_{j, k}\ra$ with $r\neq r_{j,k}$ hold when we prove the first result. From Lemma~\ref{lemma:F-yi_li'bd}, we can obtain that for all $i \in \cI_{j, k}$ and $t>T_1$, it holds
    \begin{align}\label{eq:ell_i_upper}
        -\ell_{i}'^{(t)} &\leq e\cdot \exp{\big(-F_{j}(\Wb_{j}^{(t)}, \Xb_i)\big)} \leq e\cdot \exp\Big(-\frac{1}{m}\big\la\wb_{j, r_{j, k}}^{(t)}, \bmu_{j, k}\big\ra\Big),
    \end{align}
    since the activation function $\sigma(\cdot)$ is always positive. Additionally, we can also obtain that for all $i \in \cJ_{j, k}$ and $t>T_1$, it holds
    \begin{align*}
        -\ell_{i}^{(t)} &\geq \frac{1}{2e}\exp{\big(-F_{j}(\Wb_{j}^{(t)}, \Xb_i)\big)}\\
        &\geq \frac{1}{2e}\exp{\bigg(-\frac{1}{m}\sum_{r=1}^m\Big[\sigma\big(\la\wb_{j, r}^{(t)},\bmu_{j, k}\ra\big)+\sum_{k'=1}^{P-1}\sigma\big(\la\wb_{j, r}^{(t)},\bxi_{i, k'}\ra\big)\Big]\bigg)}\\
    &= \frac{1}{2e}\exp{\bigg(-\frac{1}{m}\la\wb_{j, r_{j, k}}^{(t)} , \bmu_{j, k}\ra -\frac{1}{m}\sum_{k'=1}^{P-1}\sigma\big(\la\wb_{j, r_{j, k}}^{(t)},\bxi_{i, k'}\ra\big)\bigg)}\\
    &\quad\cdot \exp{\bigg(-\frac{1}{m}\sum_{r\neq r_{j, k}}\Big[\sigma\big(\la\wb_{j, r}^{(t)},\bmu_{j, k}\ra\big)+\sum_{k'=1}^{P-1}\sigma\big(\la\wb_{j, r}^{(t)},\bxi_{i, k'}\ra\big)\Big]\bigg)}\geq \frac{1}{2e^2}\exp\Big(-\frac{1}{m}\big\la\wb_{j, r_{j, k}}^{(t)}, \bmu_{j, k}\big\ra\Big).
    \end{align*}
    The last inequality is because
\begin{align*}
    \frac{1}{m}\sum_{k'=1}^{P-1}\sigma\big(\la\wb_{j, r_{j, k}}^{(t)},\bxi_{i, k'}\ra\big) + \frac{1}{m}\sum_{r\neq r_{j, k}}\Big[\sigma\big(\la\wb_{j, r}^{(t)},\bmu_{j, k}\ra\big)+\sum_{k'=1}^{P-1}\sigma\big(\la\wb_{j, r}^{(t)},\bxi_{i, k'}\ra\big)\Big] \leq 1
\end{align*}
by our assumption $\la\wb_{j, r}^{(t)}, \bmu_{j,k}\ra\leq \frac{1}{2K m}$ and $\la\wb_{j, r}^{(t)}, \xi_{i,k'}\ra \leq \tilde O(\sigma_0 \sigma_{\mathrm{noise}} \sqrt{d})$. With such upper and lower bounds for $\ell_i'^{(t)}$ in hands, we can provide an upper and lower bound for the iterations of $\la\wb_{j, r_{j, k}}^{(t)} , \bmu_{j, k}\ra$ as
\begin{align*}
    \la\wb_{j, r}^{(t+1)},\bmu_{j, k}\ra 
    &= \la\wb_{j, r}^{(t)},\bmu_{j, k}\ra -\frac{\eta}{nm}\sigma'\big(\la\wb_{j, r}^{(t)},\bmu_{j, k}\ra\big)\sum_{i\in \cI_{j, k}} \ell_i'^{(t)} \\
    &\leq \la\wb_{j, r}^{(t)},\bmu_{j, k}\ra + \frac{e\eta}{nm}e^{-\frac{1}{m}\la\wb_{j, r_{j, k}}^{(t)} , \bmu_{j, k}\ra}\cdot \big|\cI_{j, k}\big|\\
    &\leq \la\wb_{j, r}^{(t)},\bmu_{j, k}\ra + \frac{e\eta}{m}e^{-\frac{1}{m}\la\wb_{j, r_{j, k}}^{(t)} , \bmu_{j, k}\ra},
\end{align*}
since $\sigma'\big(\la\wb_{j, r}^{(t)},\bmu_{j, k}\ra\big) = 1$ and $\big|\cI_{j, k}\big|\leq n$, and 
\begin{align*}
    \la\wb_{j, r}^{(t+1)},\bmu_{j, k}\ra 
    &\geq \la\wb_{j, r}^{(t)},\bmu_{j, k}\ra -\frac{\eta}{nm}\sigma'\big(\la\wb_{j, r}^{(t)},\bmu_{j, k}\ra\big)\sum_{i\in \cJ_{j, k}} \ell_i'^{(t)} \\
    &\geq \la\wb_{j, r}^{(t)},\bmu_{j, k}\ra + \frac{\eta}{2e^2nm}e^{-\frac{1}{m}\la\wb_{j, r_{j, k}}^{(t)} , \bmu_{j, k}\ra}\cdot \big|\cJ_{j, k}\big|\\
    &\geq \la\wb_{j, r}^{(t)},\bmu_{j, k}\ra +
    \frac{\pi\eta}{16e^2Km}e^{-\frac{1}{m}\la\wb_{j, r_{j, k}}^{(t)} , \bmu_{j, k}\ra}.
\end{align*}
since $\sigma'\big(\la\wb_{j, r}^{(t)},\bmu_{j, k}\ra\big) = 1$ and $\big|\cJ_{j, k}\big|\geq \frac{\pi_1}{8K}$
by Lemma~\ref{lemma:numberofdataII}.
Applying these upper and lower bound on Lemma~\ref{lemma:exp_sequence}, for all $t> T_1$ we obtain that 
\begin{align}\label{eq:wmu_lower}
    \la\wb_{j, r_{j, k}}^{(t)},\bmu_{j,k} \ra&\geq m\log\Big(\frac{\eta}{16e^2Km^2}(t-T_1)+e^{\frac{\kappa}{m}}\Big)\geq m\log(t-T_1)-2m\log m + m\log \eta-O(m)
\end{align}
and 
\begin{align}\label{eq:wmu_upper}
    \la\wb_{j, r_{j, k}}^{(t)},\bmu_{j,k} \ra&\leq \frac{e\eta}{m}e^{-\frac{\kappa}{m}} + m\log\Big(\frac{e\eta}{m^2}(t-T_1)+e^{\frac{\kappa}{m}}\Big)\leq m\log(t-T_1)-2m\log m + m\log \eta+O(m).
\end{align}
This finishes the proof of the conclusion for $\la \wb_{j,r_{j, k}}^{(T^*)} , \bmu_{j, k}\ra$.  Now, we use induction to prove that $\la\wb_{j,r}^{(T^*)}, \bmu_{j, k}\ra \leq \frac{1}{2Km}$ when $r \neq r_{j, k}$. We first derive a result that will be used for the following induction proof. Plugging~\eqref{eq:wmu_lower} into~\eqref{eq:ell_i_upper}, for all $i\in \cI_{j, k}$ and $t>T_1$, we have
\begin{align*}
    -\ell_{i}'^{(t)} \leq \frac{8e^3Km^2}{\eta(t-T_{1} + 1)}
\end{align*}
Taking the sum from $t=T_{1}$ to $T^*$, we have
\begin{align}\label{eq:ell_i_upperII}
    -\sum_{t=T_1}^{T^*}\ell_{i}'^{(t)}\leq \frac{8e^3Km^2}{\eta}\log(T^*-T_{1}) \leq\tilde\Theta\Big(\frac{m^2}{\eta}\Big)
\end{align}
Since at $T_1$, we have $\la\wb_{j,r}^{(T_1)}, j\bmu_k\ra \leq \frac{1}{4Km}$ if $r\neq r_{j, k}$. Suppose that exists $T_1 < \tilde T \leq T^*$ such that $\la\wb_{j,r}^{(t)}, j\bmu_k\ra \leq \frac{1}{2Km}$ holds for all $T_1 \geq t \leq \tilde T -1$. Then by the iterative rule for $\la\wb_{j,r}^{(t)}, \bmu_{j, k}\ra$ and applying ~\eqref{eq:ell_i_upperII}, we have
\begin{align*}
    \la\wb_{j,r}^{(\tilde T)}, \bmu_{j, k}\ra
    &\leq \la\wb_{j,r}^{(\tilde T-1 )}, \bmu_{j, k}\ra - \frac{\eta}{\kappa^{q-1}nm}\Big(\frac{1}{2Km}\Big)^{q-1}\sum_{i\in \cI_{j, k}} \ell_i'^{(\tilde T-1)}\\
    &\leq \la\wb_{j,r}^{(T_1)}, \bmu_{j, k}\ra - \frac{\eta}{\kappa^{q-1}nm}\Big(\frac{1}{2Km}\Big)^{q-1}\sum_{t=T_1}^{T^*}\sum_{i\in \cI_{j, k}} \ell_i'^{(\tilde T-1)}\\
    & \leq \la\wb_{j,r}^{(T_1)}, \bmu_{j, k}\ra - \frac{\eta}{\kappa^{q-1}nm}\Big(\frac{1}{2Km}\Big)^{q-1}\frac{8e^3Km^2}{\eta}\log(T^*-T_{1})\\
    &\leq \frac{1}{4Km} + \frac{2^{6-q}e^3\log(T^*-T_1)}{\kappa^{q-1}K^{q-3}m^{q-3}} \cdot\frac{1}{4Km}\leq \frac{1}{2Km}, 
\end{align*}
where the last inequality holds since $m^{q-3}\geq \tilde\Omega(1)$ as $q>3$.
This finishes the induction proof that $\la\wb_{j,r}^{(t)}, \bmu_{j,k}\ra \leq  \frac{1}{2Km}$ for all $r \neq r_{j, k}$ and $t\leq T^*$.
Next, we proof that $\max_{j, r, i, k'}|\rho_{j, r, i, k'}^{(t)}| \leq \sigma_0\sigma_{\mathrm{noise}}\sqrt{d}$ holds for all $t<T^*$. For simplicity we denote $\phi^{(t)} = \max_{j, r, i, k'}|\rho_{j, r, i, k'}^{(t)}|$. Obviously we have $\phi^{(T_1)} \leq \sigma_0\sigma_{\mathrm{noise}}\sqrt{d}/2$, and we suppose that exists $T_1 \leq \tilde T \leq T^* $ such that $\phi^{(t)} \leq \sigma_0\sigma_{\mathrm{noise}}\sqrt{d}$ holds for all $T_1 <t<\tilde T-1$. Then by the iterative rule for $\rho_{j, r, i, k'}^{(t)}$ and plugging ~\eqref{eq:ell_i_upperII}, we have 
\begin{align*}
    \phi^{(\tilde T)}
    &\leq \phi^{(T_1)} + \frac{[6\log(m\vee n)]^{(q+1)/2}\eta\sigma_0^{q-1}\sigma_{\mathrm{noise}}^{q+1}d^{(q+1)/2}}{\kappa^{q-1}nm}\cdot \frac{8e^3Km^2}{\eta}\log(T^*-T_{1}) \\
    &\leq \phi^{(T_1)} + \frac{16e^3K[6\log(m\vee n)]^{(q+1)/2}\log(T^*-T_{1})}{\kappa^{q-1}}\cdot\frac{\sigma_{\mathrm{noise}}^{q}d^{q/2}}{n}\cdot m\sigma_0^{q-2}\cdot\frac{\sigma_0\sigma_{\mathrm{noise}}\sqrt{d}}{2}\\
    &\leq \phi^{(T_1)} + \frac{\sigma_0\sigma_{\mathrm{noise}}\sqrt{d}}{2}\leq \sigma_0\sigma_{\mathrm{noise}}\sqrt{d}
\end{align*}
where the penultimate inequality holds by $\frac{\sigma_{\mathrm{noise}}^{q}d^{q/2}}{n}\leq \tilde O(1)$ and $m = O(\sigma_0^{2-q})$, which is derived from Condition~\ref{condition:d_sigma0_eta}.
\end{proof}

\section{Proof of Corollary~\ref{crlry:main_result}}\label{section:proof_crlry}
In this section, we provide a proof for the Proposition~\ref{crlry:main_result}. 



\begin{proof}[Proof of Corollary~\ref{crlry:main_result}]
In the following, we consider the case $j=1$, and the situation for $j=-1$ is similar. Since by Theorem~\ref{thm:main_result3}, we have $\big\|\wb_{1, r_{1, k}}^{(T)}\big\|_2 = m\log T - m\log\big(\frac{m^2}{\eta}\big) + O(m)$, and $\big\|\wb_{1, r}^{(T)}\big\|_2 \leq O(\frac{1}{m}) + O(\sigma_0 d^{1/2})$ when $r\neq r_{1,k}$, then by definition of Frobenius norm, we can derive that 
\begin{align*}
    \big\|\Wb_{+1}^{(T)}\big\|_F & =  \sqrt{\sum_{k=1}^{K}\big\|\wb_{1, r_{1, k}}^{(T)}\big\|_2^2 +  \sum_{r\neq r_{1, k}}\big\|\wb_{1, r}^{(T)}\big\|_2^2} = \sqrt{K}\bigg(m\log T - m\log\Big(\frac{m^2}{\eta}\Big) + O(m)\bigg),
\end{align*}
 By the definition of $\Wb^*$, we have $\|\Wb^*\|_F= \sqrt{K}$. By these results of Frobenius norm, we can derive that
  \begin{align}\label{eq:learn_signal_diff}
     &\Bigg\|\frac{\wb_{1, r_{1, k}}^{(T)}}{\big\| \Wb_{+1}^{(T)} \big\|_F} - 
     \frac{\bmu_{1, k}}{\|\Wb^*\|_F}\Bigg\|_2 \notag\\
     =& \Bigg\|\bigg(\frac{\la\wb_{1, r_{1, k}}^{(T)}-\wb_{1, r_{1, k}}^{(0)}, \bmu_{1,k}\ra}{\big\|\Wb_{+1}^{(T)}\big\|_F }-\frac{1}{\|\Wb^*\|_F}\bigg)\bmu_{1, k}+\frac{\wb_{1, r_{1, k}}^{(0)} + \bXi_{1, r_{1, k}}^{(T)}}{\big\|\Wb_{+1}^{(T)}\big\|_F} \notag\\
     &+\sum_{k'\neq k}\frac{\la\wb_{1, r_{1, k}}^{(T)}-\wb_{1, r_{1, k}}^{(0)}, \bmu_{1, k'}\ra}{\big\|\Wb_{+1}^{(T)}\big\|_F}\bmu_{1, k'} +\sum_{k' =1}^{K}\frac{\la\wb_{1, r_{1, k}}^{(T)}-\wb_{1, r_{1, k}}^{(0)}, \bmu_{-1, k'}\ra}{\big\|\Wb_{+1}^{(T)}\big\|_F}\bmu_{-1, k'}\Bigg\|_2\notag\\
     \leq &\Bigg|\frac{\la\wb_{1, r_{1, k}}^{(T)}-\wb_{1, r_{1, k}}^{(0)}, \bmu_{1, k}\ra}{\big\|\Wb_{+1}^{(T)}\big\|_F }-\frac{1}{\|\Wb^*\|_F}\Bigg| + \frac{\|\wb_{1, r_{1, k}}^{(0)}\|_2 + \|\bXi_{1, r_{1, k}}^{(T)}\|_2}{\big\|\Wb_{+1}^{(T)}\big\|_F} \notag\\
     &+ \sum_{k'\neq k}\frac{\big|\la\wb_{1, r_{1, k}}^{(T)}-\wb_{1, r_{1, k}}^{(0)}, \bmu_{1, k'}\ra\big|}{\big\|\Wb_{+1}^{(T)}\big\|_F}+ \sum_{k'=1}^K\frac{\big|\la\wb_{1, r_{1, k}}^{(T)}-\wb_{1, r_{1, k}}^{(0)}, \bmu_{-1, k'}\ra\big|}{\big\|\Wb_{+1}^{(T)}\big\|_F}\notag\\
     \leq & O\bigg(\frac{1}{\log T}\bigg) + O\bigg(\frac{1}{m\log T}\bigg) + O\bigg(\frac{1}{m^2\log T}\bigg) + O\bigg(\frac{1}{m\sqrt{d}\log T}\bigg) \leq O\bigg(\frac{1}{\log T}\bigg).
 \end{align}
 The first inequality is from triangle inequality. The second inequality holds because $\|\wb_{1, r_{1, k}}^{(0)}\|_2\leq O(\sigma_0 d^{1/2} )\leq O(1)$; $\|\bXi_{1, r_{1, k}}^{(T)}\|_2\leq O(\sigma_0 n^{1/2})\leq O(1) $; $\la\wb_{1, r_{1, k}}^{(T)}-\wb_{1, r_{1, k}}^{(0)}, \bmu_{1, k'}\ra \leq O(1/m)$; $\la\wb_{1, r_{1, k}}^{(T)}-\wb_{1, r_{1, k}}^{(0)}, \bmu_{-1, k'}\ra \leq O(\sigma_0)\leq O(d^{-1/2})$; and
 \begin{align*}
     &\Bigg|\frac{\la\wb_{1, r_{1, k}}^{(T)}-\wb_{1, r_{1, k}}^{(0)}, \bmu_{1, k}\ra}{\big\|\Wb_{+1}^{(T)}\big\|_F}-\frac{1}{\|\Wb^*\|_F}\Bigg| \\
     \leq& \frac{1}{\sqrt{K}}\Bigg|\frac{m\log T - m\log\Big(\frac{m^2}{\eta}\Big) + O(m)}{m\log T - m\log\Big(\frac{m^2}{\eta}\Big) + O(m)}-1\Bigg|\leq O\bigg(\frac{1}{\log T}\bigg)
 \end{align*}
by Theorem~\ref{thm:main_result3}, Lemma~\ref{lemma:data_innerproducts} and Lemma~\ref{lemma:initialization_norms}. Next we consider the filters $\|\wb_{1,r}\|$ with $r\neq r_{1, k}$ and we can direct obtain
\begin{align}\label{eq:not_learn_signal_diff}
    \Bigg\|\frac{\wb_{1,r}^{(T)}}{\big\|\Wb_{+1}^{(T)}\big\|_F}\Bigg\|_2\leq O\Bigg(\frac{m^{-1} + \sigma_0 d^{1/2}}{\sqrt{K}m\log T}\Bigg)\leq O\bigg(\frac{1}{m\log T}\bigg),
\end{align}
since $\big\|\wb_{1, r}^{(T)}\big\|_2 \leq O(m^{-1} + \sigma_0 d^{1/2})$ in Theorem~\ref{thm:main_result3}. We find a permutation matrix $\Pb_{+}$ such that the filters $\wb_{1, 1}, \cdots, \wb_{1, K}$ are arranged to the first $K$ rows, i.e., $\Pb_{+} \cdot \Wb_{+1}^{(T)} = [\wb_{1, r_{1, 1}},\wb_{1, r_{1, 2}},\cdots, \wb_{1, r_{1, K}},\cdots]^\top$, then by triangle inequality and preceding results~\eqref{eq:learn_signal_diff} and~\eqref{eq:not_learn_signal_diff}, we finally derive that
\begin{align*}
    &\bigg\|\frac{\Wb^*}{\| \Wb^* \|_F}- \Pb_+\frac{\Wb_{+1}^{(T)}}{ \big\| \Wb_{+1}^{(T)} \big\|_F} \bigg\|_F\\
    \leq &\sum_{k=1}^K\Bigg\|\frac{\wb_{1, r_{1, k}}^{(T)}}{\big\| \Wb_{+1}^{(T)} \big\|_F} - 
     \frac{\bmu_{1, k}}{\|\Wb^*\|_F}\Bigg\|_2 + \sum_{r\neq r_{1, k}}\bigg\|\frac{\wb_{1,r}^{(T)}}{\big\|\Wb_{+1}^{(T)}\big\|_F}\bigg\|_2\leq O\bigg(\frac{1}{\log T}\bigg).
\end{align*}
This completes the proof.
\end{proof}

\section{Technical lemmas}
\subsection{Concentration results}\label{section:concentration}
\begin{lemma}\label{lemma:numberofdata}
Suppose that $\delta > 0$, then for any $\cI \subseteq [n]$, with probability at least $1 - \delta$, 
\begin{align*}
    |\{i \in \cI : y_i = 1\}|,~ |\{i \in \cI :y_i = -1\}| = \frac{|\cI |}{2} + O\Big(\sqrt{|\cI |\log (1/\delta)}\Big).
\end{align*}
\end{lemma}
\begin{proof}[Proof of Lemma~\ref{lemma:numberofdata}] By Hoeffding's inequality, with probability at least $1 - O(\delta)$, we have
\begin{align*}
    \Bigg| \frac{1}{|\cI |}\sum_{i \in \cI } \ind\{y_i = 1\}  - \frac{1}{2} \Bigg| \leq O\Bigg(\sqrt{\frac{\log(1/\delta)}{|\cI |}}\Bigg).
\end{align*}
Therefore,
\begin{align*}
    |\{i \in \cI : y_i = 1\}| = \sum_{i\in \cI } \ind\{y_i = 1\} = \frac{|\cI |}{2} + O\Big(\sqrt{|\cI |\log (1/\delta)}\Big).
\end{align*}
This proves the result for $|\{i \in \cI : y_i = 1\}|$. The proof for $|\{i \in \cI : y_i = -1\}|$ is exactly the same, and we can conclude the proof by applying a union bound. 
\end{proof}

\begin{lemma}\label{lemma:numberofdataII}
Suppose that $\delta >0$, then for 
$\cJ_{j, k}$ defined in Section~\ref{section:proof_I}, with probability at least $1 - \delta$, it holds that
\begin{align*}
    |\cJ_{j, k}| = \frac{\pi}{K}n + O\Big(\sqrt{n\log (1/\delta)}\Big), 
\end{align*}
where $\pi$ is a constant solely depending on $K$.
\end{lemma}
\begin{proof}[Proof of Lemma~\ref{lemma:numberofdataII}] 
Let $\cI_{+1} = \{i|y_i = 1, i\in [n]\}$ and $\cI_{-1} = \{i|y_i = -1, i\in [n]\}$.
By Lemma~\ref{lemma:numberofdata}, it holds that $\big|\cI_{+1}\big|, \big|\cI_{-1}\big|\geq \frac{n}{4}$. By the definition of distribution $\cD$ in Definition~\ref{def:data}, there exists a positive probability $\pi = \pi(K)$ such that $\PP(s =1) =\pi$. (We remind the reader that $s$ is the number of object patches in each input.) Then by Hoeffding's inequality, with probability at least $1 - \delta$, we have
\begin{align*}
    \Bigg| \frac{1}{|\cI_{j}|}\sum_{i \in \cI_{j}} \ind\{i\in \cJ_{j, k}\}  - \frac{\pi}{K} \Bigg| \leq O\Bigg(\sqrt{\frac{\log(1/\delta)}{n}}\Bigg).
\end{align*}
Therefore,
\begin{align*}
    |\cJ_{j, k}| = \sum_{i \in \cI_{j}} \ind\{i\in \cJ_{j, k}\} = \frac{\pi}{K}|\cI_{j}| + O\Big(\sqrt{n\log (1/\delta)}\Big) \geq \frac{\pi n}{8K},
\end{align*}
where the last inequality is from Lemma~\ref{lemma:numberofdata}. This finishes the proof.
\end{proof}

\begin{lemma}\label{lemma:numberofdataIII}
Suppose that $\delta >0$, then for 
$K_{1, k}$ and $K_{-1, k}$ defined in Section~\ref{section:proof_prop}, 
\begin{align*}
    K_{1, k}, K_{-1, k}= \frac{\sum_{i=1}^{n}s_i}{2K} + O\Big(\sqrt{n\log (1/\delta)}\Big).
\end{align*}
holds with probability at least $1 - \delta$ for all $k\in[K]$.
\end{lemma}
\begin{proof}[Proof of Lemma~\ref{lemma:numberofdataIII}] 
Similar to the definition of number of total object patches, we can also denote $\sum_{y_i=1}s_i$ the number of signal patches from inputs with positive label, and $\sum_{y_i=-1}s_i$ the number of object patches from inputs with negative label. By Lemma~\ref{lemma:numberofdata} and the fact $P\leq O(1)$, with probability at least $1 - \delta/2$, we have
\begin{align*}
    \sum_{y_i=1}s_i, \sum_{y_i=-1}s_i = \frac{\sum_{i=1}^{n}s_i}{2} + O\Big(\sqrt{n\log (1/\delta)}\Big).
\end{align*}
Besides, for all $k \in [K]$, by Hoeffding's inequality, with probability at least $1 - \delta/2$, we have
\begin{align*}
    \Bigg| K_{1, k} - \frac{\sum_{y_i=1}s_i}{K}\Bigg|, \Bigg| K_{-1, k} - \frac{\sum_{y_i=-1}s_i}{K}\Bigg| \leq O\Big(\sqrt{n\log (1/\delta)}\Big).
\end{align*}
Applying a union bound over the results from two inequalities above completes the proof.
\end{proof}

\begin{lemma}\label{lemma:absgaussian}
Suppose that $z\sim
 \cN(0,1)$, then $\PP(|z|\leq t) = O\rbr{t}$.
\end{lemma}
\begin{proof}[Proof of Lemma~\ref{lemma:absgaussian}] 
We use $\phi(x)$ to denote the density function of the standard Gaussian random variable, and then we know that $\max \phi(x) = \phi(0)$. By this fact,
\begin{align*}
    \PP(|z|\leq t) = 2\int_0^t \phi(x)dx \leq 2\phi(0)t = O\rbr{t},
\end{align*}
which completes the proof.
\end{proof}

\begin{lemma}\label{lemma:data_innerproducts}
Suppose that $\delta > 0$ and $d = \Omega\big(\log(nm / \delta)\big)$. Then with probability at least $1 - O(\delta)$, it holds that
\begin{align*}
    &\| \bxi_{i, k} \|_2^2 = \Theta(\sigma_{\mathrm{noise}}^2 d);\\
    &\| \wb_{j, r}^{(0)} \|_2^2 = \Theta(\sigma_{0}^2 d);\\
    & \big|\la \bxi_{i, k}, \bxi_{i', k'} \ra\big| \leq O\big(\sigma_{\mathrm{noise}}^2 \cdot \sqrt{d \log(n^2 / \delta)}\big)
\end{align*}
for all $j\in \{+1, -1\}$, $r\in [m]$, and all $i, i' \in [n]$, $k \in [P-s_i], k'\in [P- s_{i'}]$ such that $\{i, k\} \neq \{i', k'\}$.
\end{lemma}
\begin{proof}[Proof of Lemma~\ref{lemma:data_innerproducts}] By Bernstein's inequality, with probability at least $1 - O(\delta /n)$ we have
\begin{align*}
    \big| \| \bxi_{i,k} \|_2^2 - \sigma_{\mathrm{noise}}^2 d \big| = O\big(\sigma_{\mathrm{noise}}^2 \cdot \sqrt{d \log(n / \delta)}\big).
\end{align*}
Therefore, as long as $d = \Omega\big(\log(n / \delta)\big)$, we have
\begin{align*}
    \| \bxi_{i,k} \|_2^2 = \Theta(\sigma_{\mathrm{noise}}^2 d).
\end{align*}
Similarly, by Bernstein's inequality, with probability at least $1 - O(\delta /m)$ we have
\begin{align*}
    \big| \|\wb_{j, r}^{(0)} \|_2^2 - \sigma_{0}^2 d \big| = O\big(\sigma_{\mathrm{noise}}^2 \cdot \sqrt{d \log(m / \delta)}\big).
\end{align*}
Therefore, as long as $d = \Omega\big(\log(m / \delta)\big)$, we have
\begin{align*}
    \| \wb_{j, r}^{(0)} \|_2^2 = \Theta(\sigma_{0}^2 d).
\end{align*}
Moreover,  for any $i, i', k, k'$ with $\{i, k\} \neq \{i', k'\}$, clearly $\la \bxi_{i, k}, \bxi_{i', k'} \ra$ has mean zero and by Bernstein's inequality, with probability at least $1 - O(\delta / n^2)$ we have
\begin{align*}
    | \la \bxi_{i,k}, \bxi_{i', k'} \ra| \leq O\big(\sigma_{\mathrm{noise}}^2 \cdot \sqrt{d \log(n^2/\delta)}\big).
\end{align*}
Applying a union bound completes the proof.
\end{proof}

\begin{lemma}\label{lemma:initialization_norms}
 Suppose that $d \geq \Omega\big(\log(mn/\delta)\big)$, $ m = \Omega\big(\log(1 / \delta)\big)$. Then with probability at least $1 - O(\delta)$, it holds that
 \begin{align*}
    &\big|\la \wb_{j,r}^{(0)}, \bmu_{j, k} \ra \big| = O\rbr{\sqrt{\log(m/\delta)} \cdot \sigma_0 },\\
    &\big| \la \wb_{j,r}^{(0)}, \bxi_{i, k'} \ra \big| = O\rbr{\sqrt{ \log(mn/\delta)}\cdot \sigma_0 \sigma_{\mathrm{noise}} \sqrt{d}}
\end{align*}
for all $r\in [m]$,  $j\in \{\pm 1\}$, $i\in [n]$, $k\in [K]$ and $k' \in [P-s_i]$. Besides,
\begin{align*}
    &\max_{r\in[m]}  \la \wb_{j,r}^{(0)}, \bmu_{j, k} \ra = \Omega(\sigma_0),\\
    &\max_{r\in[m]} \la \wb_{j,r}^{(0)}, \bxi_{i, k'} \ra = \Omega\rbr{
    \sigma_0 \sigma_{\mathrm{noise}} \sqrt{d}}
\end{align*}
for all $j\in \{\pm 1\}$, $i\in [n]$, $k\in [K]$ and $k' \in [P-s_i]$. Moreover,
\begin{align*}
    \la \wb_{j,r}^{(0)}, \bmu_{j,k} \ra \rbr{1+\Theta\rbr{\frac{\delta}{m^2}}} \leq \max_{r\in [m]} \la\wb_{j,r}^{(0)}, \bmu_{j, k} \ra
\end{align*}
for all $r\neq \argmax_{r\in [m]} \la\wb_{j,r}^{(0)}, \bmu_{j, k} \ra$, $j\in \{\pm 1\}$ and $k \in [K]$.
\end{lemma}
\begin{proof}[Proof of Lemma~\ref{lemma:initialization_norms}] 
It is clear that for each $r\in [m]$, $\la \wb_{j,r}^{(0)}, \bmu_{j, k} \ra$ is a Gaussian random variable with mean zero and variance $\sigma_0^2$. Therefore, by Gaussian tail bound and union bound, with probability at least $1 - O(\delta)$, 
\begin{align*}
    \big|\la \wb_{j,r}^{(0)}, \bmu_{j,k} \ra\big| \leq O\big(\sqrt{\log(m/\delta)} \cdot \sigma_0 \big).
\end{align*}
Moreover, $\PP\big( \sigma_0 \| \bmu \|_2 / 2 > \la \wb_{j,r}^{(0)}, \bmu_{j, k} \ra \big)$ is an absolute constant, and therefore by the condition on $m$, we have
\begin{align*}
    \PP\big( \sigma_0  / 2 \leq \max_{r\in[m]} \la \wb_{j,r}^{(0)}, \bmu_{j, k} \ra \big) &= 1 - \PP( \sigma_0  / 2 > \max_{r\in[m]} \la \wb_{j,r}^{(0)}, \bmu_{j, k} \ra  \big) \\
    &= 1 - \PP\big( \sigma_0  / 2 > \la \wb_{j,r}^{(0)}, \bmu_{j, k} \ra \big)^{2m} \\
    &\geq 1 - O(\delta).
\end{align*}
By Lemma~\ref{lemma:data_innerproducts}, with probability at least $1 - O(\delta)$, $\| \bxi_{i, k'} \|_2 = \Theta\big(\sigma_{\mathrm{noise}} \sqrt{d}\big)$ for all $i\in [n]$ and $k' \in [P-s_i]$. Therefore, the result for $\la \wb_{j,r}^{(0)}, \bxi_{i, k'}\ra$ follows the same proof as $\la \wb_{j,r}^{(0)}, \bmu_{j, k} \ra$.\\ 
Lastly, for different $r, r'$ and $\forall t > 0$,
\begin{align*}
\PP\rbr{\frac{\big|\la \wb_{j,r}^{(0)}, \bmu_{j, k} \ra\big| \vee \big|\la \wb_{j,r'}^{(0)}, \bmu_{j, k} \ra\big|}{\big|\la \wb_{j,r}^{(0)}, \bmu_{j, k} \ra\big| \wedge \big|\la \wb_{j,r'}^{(0)}, \bmu_{j, k} \ra\big|} \leq 1 + t} &\leq \PP\rbr{1-t\leq \frac{\big|\la \wb_{j,r}^{(0)}, \bmu_{j, k} \ra\big| }{ \big|\la \wb_{j,r'}^{(0)}, \bmu_{j, k}\ra\big|} \leq 1 + t} \\
&\leq \PP\rbr{\big|\la \wb_{j,r}^{(0)}, \bmu_{j, k} \ra\big|\leq 2t \big|\la \wb_{j,r'}^{(0)}, \bmu_{j, k}\ra\big|} = O(t),\end{align*}
where the last equality holds from Lemma~\ref{lemma:absgaussian} and the fact that $\la \wb_{j,r}^{(0)}, \bmu_{j, k} \ra$ and $\la \wb_{j,r'}^{(0)}, \bmu_{j, k}\ra$ are independent Gaussian random variables with mean $0$ and same variance. By this result, let $t = \Theta(\frac{\delta}{m^2})$ and use union bound, we could deduce that with probability at least $1 - O(\delta)$,
\begin{align*}
    &\la \wb_{j,r}^{(0)}, \bmu_{j, k} \ra \rbr{1+\Theta\rbr{\frac{\delta}{m^2}}} \leq \max_{r\in [m]} \la\wb_{j,r}^{(0)}, \bmu_{j, k} \ra
\end{align*}
for all $r\neq \argmax_{r\in [m]} \la\wb_{j,r}^{(0)}, \bmu_{j, k} \ra $.
\end{proof}

\subsection{Tensor power methods}
The following lemmas are inspired by~\citet{allen2022towards, jelassi2022vision}
\begin{lemma}\label{lemma:sequence}
If a positive sequence $\{x_t\}_{t=0}^{\infty}$ satisfies the updating rules $x_{t+1} = x_{t} + \eta \cdot C_t \cdot x_{t}^{q-1}$, then $\forall k\in \NN, \zeta \in (0,1)$, we have
\begin{align}\label{eq:seq_lemma1}
    \sum_{t>0, x_t \leq (1+\zeta)^k x_0}\eta C_t \leq \frac{\zeta}{x_0^{q-2}}\frac{1-\frac{1}{(1+\zeta)^{(q-2)k}}}{1-\frac{1}{(1+\zeta)^{(q-2)}}} + \eta \cdot\sbr{(1+\zeta)^{q-1}\sum_{g=0}^{k-1} C_{\cT_{g+1}-1} + C_{\cT_{k}}},
\end{align}
and 
\begin{align}\label{eq:seq_lemma2}
    \sum_{t>0, x_t \leq (1+\zeta)^k x_0}\eta C_t \geq \frac{\zeta}{x_0^{q-2}(1+\zeta)^{q-1}}\frac{1-\frac{1}{(1+\zeta)^{(q-2)k}}}{1-\frac{1}{(1+\zeta)^{(q-2)}}} - \frac{\eta}{(1+\zeta)^{q-1}}\sum_{g=1}^{k-1}C_{\cT_{g}-1},
\end{align}
where $\cT_{g}$ be the first iteration such that $x_{t} \geq (1+\zeta)^g x_0$
\end{lemma}

\begin{proof}[Proof of Lemma~\ref{lemma:sequence}]
By the definition of $\cT_{g}$, we have
\begin{align}\label{eq:seq_ine1}
    x_{\cT_{g+1}} - x_{\cT_{g}} = \sum_{t \in [\cT_{g},\cT_{g+1})} \eta C_t x_t^{q-1} \geq \sum_{t \in [\cT_{g},\cT_{g+1})} \eta \cdot C_t \cdot \big[x_0(1+\zeta)^g\big]^{q-1},
\end{align}
and
\begin{align}\label{eq:seq_ine2}
    x_{\cT_{g+1}} - x_{\cT_{g}} &= x_{\cT_{g+1}-1} - x_{\cT_{g}} + \eta\cdot C_{\cT_{g+1}-1}\cdot x_{\cT_{g+1}-1}^{q-1} \notag\\
    &\leq \zeta(1+\zeta)^g x_0 + \eta \cdot C_{\cT_{g+1}-1} \cdot\big[x_0(1+\zeta)^{g+1}\big]^{q-1}.
\end{align}
By combining \eqref{eq:seq_ine1} and \eqref{eq:seq_ine2} in order and rearranging some items, we could deduce,
\begin{align*}
    \sum_{t \in [\cT_{g},\cT_{g+1})} \eta C_t \leq
    \frac{\zeta}{x_0^{q-2}[(1+\zeta)^{q-2}]^g} + \eta(1+\zeta)^{q-1} C_{\cT_{g+1}-1}.
\end{align*}
Take a telescoping sum of this result, and then we finish the proof of \eqref{eq:seq_lemma1}. For \eqref{eq:seq_lemma2}, considering the opposite direction of the inequalities \eqref{eq:seq_ine1} and \eqref{eq:seq_ine2}, repeating the previous process will get the result.
\end{proof}

\begin{lemma}\label{lemma:sequence_compare}
Suppose there are two positive sequence $\{x_t\}_{t=0}^{\infty}$ and $\{y_t\}_{t=0}^{\infty}$ satisfying the following updating rules:
\begin{align*}
    & x_{t+1} = x_{t} + \eta \cdot C_t \cdot {x_t}^{q-1};\\
    & y_{t+1} = y_{t} + \eta \cdot C_t \cdot {y_t}^{q-1},
\end{align*}
with $q \geq 3$ and $\frac{x_0}{y_0} \geq 1 + c$, where $c$ is a small positive number. For any two positive number $A_x$ and $A_y$, let $T_x$, $T_y$ are the first time s.t. $x_{T_x} \geq A_x$ and $y_{T_y} \geq A_y$ respectively. If we have $0 < C_t < \bar C$ and $\eta$ and $y_0$ are both sufficiently small such that $\eta = \tilde O\rbr{\frac{c}{\bar C y_0^{q-3} A_y}}$ and $\frac{y_0}{A_y} \leq O(c)$, then it holds that $T_x \leq T_y$.
\end{lemma}

\begin{proof}[Proof of Lemma~\ref{lemma:sequence_compare}]
    For a positive $\zeta > 0$, let $k_1, k_2$ be the smallest integer s.t. $x_0(1+\zeta)^{k_1} \geq A_x$ and $y_0(1+\zeta)^{k_2} \geq A_y$. From these definitions, we have
    \begin{align*}
        \frac{\log(\frac{A_x}{x_0})}{\log(1+\zeta)}\leq k_1 < \frac{\log(\frac{A_x}{x_0})}{\log(1+\zeta)}+1,
    \end{align*} and
    \begin{align*}
        \frac{\log(\frac{A_y}{y_0})}{\log(1+\zeta)}\leq k_2 < \frac{\log(\frac{A_y}{y_0})}{\log(1+\zeta)}+1.
    \end{align*}
    By Lemma~\ref{lemma:sequence}, we further derive that
    \begin{align}\label{eq:seq_xt}
    \sum_{t=0}^{T_x}\eta C_t &\leq \frac{\zeta}{x_0^{q-2}}\frac{1-\frac{1}{(1+\zeta)^{(q-2)k_1}}}{1-\frac{1}{(1+\zeta)^{(q-2)}}} + \eta \cdot \sbr{(1+\zeta)^{q-1}\sum_{g=0}^{k_1-1} C_{\cT_{g+1}-1} + C_{\cT_{k_1}}}\notag\\
    &\leq \frac{\zeta}{x_0^{q-2}}\frac{1}{1-\frac{1}{(1+\zeta)^{(q-2)}}} + \eta\cdot (1+\zeta)^{q-1}(k_1+1) \bar C,
\end{align}
and
\begin{align}\label{eq:seq_zt}
    \sum_{t=0}^{T_y} \eta C_t &\geq \frac{\zeta}{y_0^{q-2}(1+\zeta)^{q-1}}\frac{1-\frac{1}{(1+\zeta)^{(q-2)k_2}}}{1-\frac{1}{(1+\zeta)^{(q-2)}}} - \frac{\eta}{(1+\zeta)^{q-1}}\sum_{g=1}^{k_2-1}C_{\cT_{g}-1}\notag\\
    &\geq \frac{\zeta}{y_0^{q-2}(1+\zeta)^{q-1}}\frac{1-\rbr{\frac{y_0}{A_y}}^{q-2}}{1-\frac{1}{(1+\zeta)^{(q-2)}}} - \frac{\eta}{(1+\zeta)^{q-1}}(k_2-1)\bar C.
\end{align}
We use \eqref{eq:seq_zt} minus \eqref{eq:seq_xt} and get 
\begin{align*}
    \sum_{t=0}^{T_y} \eta C_t - \sum_{t=0}^{T_x} \eta C_t 
    &\geq \underbrace{\frac{\zeta}{1-\frac{1}{(1+\zeta)^{(q-2)}}}\cbr{\frac{1-\rbr{\frac{y_0}{A_y}}^{q-2}}{y_0^{q-2}(1+\zeta)^{q-1}} - \frac{1}{x_0^{q-2}}}}_{I_{1}} \notag\\
    &\qquad -\underbrace{\eta \bar C \cbr{\frac{k_2-1}{(1+\zeta)^{q-1}} +  (1+\zeta)^{q-1}(k_1+1)}}_{I_{2}}.
\end{align*}
We consider the value of $I_1$ and $I_2$ separately and carefully choose $\zeta$ such that
\begin{align*}
    (1+\zeta)^{q-1} = \rbr{1-\frac{y_0}{A_y}}^{2} \rbr{1+c}^{q-2} = 1+\Theta\rbr{c}.
\end{align*}
The last equality is from our assumption $\frac{y_0}{A_y} =  O\rbr{c}$, and we could also conclude $\zeta = \Theta\rbr{c}$. Then for $I_1$, we have,
\begin{align}\label{eq:seq_I1}
    I_1 &\geq \frac{\zeta}{y_0^{q-2}(1-\frac{1}{(1+\zeta)^{(q-2)}})\rbr{1+c}^{q-2}}\cbr{\frac{1- \frac{y_0}{A_y}}{\rbr{1-\frac{y_0}{A_y}}^{2}} - 1}= \Omega\rbr{\frac{1}{y_0^{q-3}A_y}}.
\end{align}
Because $ \frac{1}{1-\frac{1}{(1+\zeta)^{(q-2)}}}= \Theta\rbr{\zeta}$. For $I_2$, we have,
\begin{align}\label{eq:seq_I2}
    I_2 \leq \eta \bar C \Theta\rbr{k_1\vee k_2}= \eta \bar C\tilde\Theta\rbr{\frac{1}{c}}.
\end{align}
Now by combining \eqref{eq:seq_I1} and \eqref{eq:seq_I2}, we could conclude that $\sum_{t=0}^{T_y} \eta C_t - \sum_{t=0}^{T_x} \eta C_t \geq 0$, which completes the proof.
\end{proof}

\begin{lemma}\label{lemma:sequence_time}
Suppose a positive sequence $\{x_t\}_{t=0}^\infty$ satisfies the following iterative rules:
\begin{align*}
    &x_{t+1} \geq x_{t} + \eta \cdot C_1 \cdot{x_t}^{q-1};\\
    &x_{t+1} \leq x_{t} + \eta \cdot C_2 \cdot {x_t}^{q-1},\\
\end{align*}
with $C_2 \geq C_1 >0$. For any $v > x_0$, let $T_v$ be the first time such that $x_t \geq v$, then for any constant $\zeta > 0$, we have
\begin{align}\label{eq:seq_Tu}
    T_v \leq \frac{1+\zeta}{\eta C_1 x_0^{q-2}} + \frac{(1+\zeta)^{q-1}C_2 \log (\frac{v}{x_0}) }{C_1},
\end{align} and 
\begin{align}\label{eq:seq_Tl}
    T_v \geq \frac{1}{(1+\zeta)^{q-1}\eta C_2 x_0^{q-2}} - \frac{ \log (\frac{v}{x_0})}{(1+\zeta)^{q-2}}.
\end{align}

\end{lemma}

\begin{proof}[Proof of Lemma~\ref{lemma:sequence_time}]
Similar to the definition in Lemma~\ref{lemma:sequence}, let $\cT_{g}$ be the first iteration such that $x_{t} \geq (1+\zeta)^g x_0$. Moreover, let $g^*$ be the smallest integer such that $(1+\zeta)^{g^*} x_0 \geq v$, resulting 
\begin{align*}
    \frac{\log(\frac{v}{x_0})}{\log(1+\zeta)} \leq g^* < \frac{\log(\frac{v}{x_0})}{\log(1+\zeta)} + 1.
\end{align*}
For $t = \cT_1$, 
\begin{align*}
    x_{\cT_1} \geq x_0 + \sum_{t=0}^{\cT_1 - 1}\eta C_1 x_t^{q-1}\geq x_0 + \cT_1\eta C_1 x_0^{q-1},
\end{align*}
and we could obtain that
\begin{align}\label{eq:cT1_upper1}
    \cT_1 \leq \frac{x_{\cT_1}-x_0}{\eta C_1 x_0^{q-1}}.
\end{align}
Consider the upper-bound iteration of $x_{\cT_1}$ and the fact that $x_{\cT_1-1}\leq x_0(1+\zeta)$, we could get
\begin{align}\label{eq:xcT1_upper}
    x_{\cT_1} \leq x_{\cT_1-1} + \eta C_2 x_{\cT_1-1}^{q-1} \leq x_0(1+\zeta) + \eta C_2 x_0^{q-1}(1+\zeta)^{q-1}.
\end{align}
Combining the results from  \eqref{eq:cT1_upper1} and \eqref{eq:xcT1_upper}, we obtain that,
\begin{align*}
    \cT_1 \leq \frac{\zeta}{\eta C_1 x_0^{q-2}} + \frac{(1+\zeta)^{q-1}C_2}{C_1}.
\end{align*}
Similarly for $g >1$, 
\begin{align}\label{eq:xcTg_lower}
    x_{\cT_g} \geq x_{\cT_{g-1}}  + \sum_{t=\cT_{g-1}}^{\cT_g - 1}\eta C_1 x_t^{q-1} \geq  x_{\cT_{g-1}} + \eta C_1(\cT_g -\cT_{g-1}) x_0^{q-1}(1+\zeta)^{(g-1)(q-1)},
\end{align}
and we could bound the difference $x_{\cT_g}- x_{\cT_{g-1}}$ by the following formula,
\begin{align}\label{eq:xcTg_upper}
    x_{\cT_g}- x_{\cT_{g-1}} \leq x_{\cT_g-1} + \eta C_2 x_{\cT_g}^{q-1} - x_{\cT_{g-1}}\leq \zeta (1+\zeta)^{g-1}x_0 + \eta C_2 x_0^{q-1} (1+\zeta)^{g(q-1)}.
\end{align}
Combining the results from  \eqref{eq:xcTg_lower} and \eqref{eq:xcTg_upper}, we obtain that,
\begin{align}\label{eq:cTg_upper}
    \cT_g \leq \cT_{g-1} + \frac{\zeta}{\eta C_1 x_0^{q-2}(1+\zeta)^{(g-1)(q-2)}} + \frac{(1+\zeta)^{q-1}C_2}{C_1}.
\end{align}
Taking a telescoping sum of the results of \eqref{eq:cTg_upper} from $g=1$ to $g=g^*$ and by the fact that $T_v \leq \cT_{g^*}$, we finally get \eqref{eq:seq_Tu}. Now consider another side, similarly for $t = \cT_1$, we have
\begin{align*}
    x_{\cT_1} \leq x_0 + \sum_{t=0}^{\cT_1 - 1}\eta C_2 x_t^{q-1}\leq x_0 + \cT_1\eta C_2 x_0^{q-1}(1+\zeta)^{q-1}.
\end{align*}
Substitute that $x_{\cT_1}-x_0 \geq \zeta x_0$, we get
\begin{align}\label{eq:xcT1_lower}
    \cT_1 \geq \frac{\zeta}{\eta C_2 x_0^{q-2} (1+\zeta)^{q-1}}.
\end{align}
For $g > 1$, similarly we could derive,
\begin{align}\label{eq:xcTg_upper2}
    x_{\cT_g} \leq x_{\cT_{g-1}}  + \sum_{t=\cT_{g-1}}^{\cT_g - 1}\eta C_2 x_t^{q-1} \leq  x_{\cT_{g-1}} + \eta C_2(\cT_g -\cT_{g-1}) x_0^{q-1}(1+\zeta)^{g(q-1)}
\end{align}
and we could also lower bound the difference $x_{\cT_g}- x_{\cT_{g-1}}$ by
\begin{align}\label{eq:xcTg_lower2}
    x_{\cT_g} - x_{\cT_{g-1}} \geq x_{\cT_g} -  x_{\cT_{g-1} -1} - \eta C_2 x_{\cT_{g-1} -1}^{q-1} \geq \zeta (1+\zeta)^{g-1}x_0 - \eta C_2 x_0^{q-1} (1+\zeta)^{(g-1)(q-1)}.
\end{align}
Combining the results from  \eqref{eq:xcTg_upper2} and \eqref{eq:xcTg_lower2}, we obtain that,
\begin{align}\label{eq:cTg_lower}
    \cT_g \geq \cT_{g-1} + \frac{\zeta}{\eta C_2 x_0^{q-2}(1+\zeta)^{g(q-2) +1}} -  \frac{1}{(1+\zeta)^{q-1}}.
\end{align}
Taking a telescoping sum of the results of \eqref{eq:cTg_lower} from $g=1$ to $g=g^*-1$ and by the fact that $T_v \geq \cT_{g^*-1}$, we finally get \eqref{eq:seq_Tl}. 
\end{proof}

\subsection{Sequence increase bounds}
The following lemma characterizes the increase of positive sequence with exponential factor. Similar results are also applied in \citet{cao2023implicit, mengbenign2024}.
\begin{lemma}\label{lemma:exp_sequence}
Suppose that a positive sequence $x_t$, $t\geq 0$ follows the iterative formula
\begin{align*}
    x_{t+1} = x_t + c_1 e^{- c_2 x_t}
\end{align*}
for some $c_1, c_2>0$. 
Then it holds that
\begin{align*}
\frac{1}{c_2}\log(c_1 c_2 t + e^{c_2 x_0}) \leq x_t \leq c_1 e^{- c_2 x_0} + \frac{1}{c_2}\log(c_1 c_2 t + e^{c_2 x_0})
\end{align*}
for all $t\geq 0$.
\end{lemma}
\begin{proof}[Proof of Lemma~\ref{lemma:exp_sequence}]
We first show the lower bound of $x_t$. Consider a continuous-time sequence $\underline x_t $, $t\geq 0$ defined by the integral equation with the same initialization.
\begin{align}\label{eq:integral_equation}
    \underline x_t = \underline x_0 + c_1 \cdot \int_{0}^{t} e^{-c_2\underline x_{\tau}}\mathrm{d} \tau,\quad \underline x_0 = x_0.   
\end{align}
Note that $ \underline x_t$ is obviously an increasing function of $t$. Therefore we have
\begin{align*}
    \underline  x_{t+1} &= \underline x_{t} + c_1\cdot \int_{t}^{t+1} e^{-c_2\underline x_{\tau}}\mathrm{d} \tau\\
    &\leq \underline x_{t} + c_1\cdot \int_{t}^{t+1} e^{-c_2\underline x_{t}}\mathrm{d} \tau\\
    &= \underline x_{t} + c_1 \exp(-c_2\underline x_{t} )
\end{align*}
for all $t\in \NN$. 
Comparing the above inequality with the iterative formula of $\{x_t\}$, we conclude by the comparison theorem that $x_t \geq \underline x_t$ for all $t\in \NN$. Note that \eqref{eq:integral_equation} has an exact solution
\begin{align*}
     \underline x_t = \frac{1}{c_2}\log(c_1 c_2 t + e^{c_2 x_0})
\end{align*}
Therefore we have
\begin{align*}
     x_t \geq \frac{1}{c_2}\log(c_1 c_2 t + e^{c_2 x_0})
\end{align*}
for all $t\in \NN$, which completes the first part of the proof. Now for the upper bound of $x_t$, we have
\begin{align*}
    x_t &= x_0 + c_1 \cdot \sum_{\tau=0}^{t-1} e^{-c_2 x_\tau} \\
    &\leq x_0 + c_1 \cdot \sum_{\tau=0}^{t} e^{- \log(c_1 c_2 \tau + e^{c_2 x_0}) }\\
    &= x_0 +c_1\cdot \sum_{\tau=0}^{t} \frac{1}{ c_1 c_2 \tau + e^{c_2 x_0}}\\
    &= x_0 + \frac{c_1}{e^{c_2 x_0}} +  c_1\cdot \sum_{\tau=1}^{t} \frac{1}{c_1 c_2 \tau + e^{c_2 x_0}}\\
    &\leq x_0 + \frac{c_1}{e^{c_2 x_0}} +  c_1\cdot \int_{0}^t \frac{1}{c_1 c_2 \tau + e^{c_2 x_0}}\mathrm{d}\tau, 
\end{align*}
where the second inequality follows by the lower bound of $x_t$ as the first part of the result of this lemma. Therefore we have
\begin{align*}
    x_t &\leq x_0 + \frac{c_1}{e^{c_2 x_0}} + \frac{1}{c_2}\log(c_1 c_2 t + e^{c_2 x_0}) - \frac{1}{c_2}\log(e^{c_2 x_0})\\
    &= c_1 e^{- c_2 x_0} + \frac{1}{c_2}\log(c_1 c_2 t + e^{c_2 x_0})
\end{align*}
This finishes the proof.
\end{proof}

\subsection{Singular value lemmas}

\begin{lemma}[Corollary 5.35 in \citet{vershynin2010introduction}]\label{lemma:singular_value_Gaussian}
Let $\Ab$ be an $d_1 \times d_2$ matrix
whose entries are independent standard normal random variables. Then for every $\delta \geq 0$,
with probability at least $1 - 2 \exp(-\delta^2/2)$, one has
\begin{align*}
    \sqrt{d_1} - \sqrt{d_2} - \delta \leq \lambda_{\min}(\Ab) \leq \lambda_{\max}(\Ab) \leq \sqrt{d_1} + \sqrt{d_2} + \delta,
\end{align*}
where $\lambda_{\min}(\Ab)$ indicates the smallest singular value of $\Ab$ and $ \lambda_{\max}(\Ab)$ indicates the largest singular value of $\Ab$.
\end{lemma}

\begin{lemma}\label{lemma:singular_value_concatenation}
Let $\Ab=[\ab_1, \cdots, \ab_{P_1}] \in \RR^{d\times P_1}$ and $\Bb=[\bb_1, \cdots, \bb_{P_2}] \in \RR^{d\times P_2}$ with $d\gg P_1, P_2$. Additionally, the columns in $\Ab$ are perpendicular to columns in $\Bb$, i.e.  $\la\ab_{p_1},\bb_{p_2}\ra = 0$ holds for all $p_1 \in [P_1]$ and $p_2 \in [P_2]$. Consider concatenating $\Ab, \Bb$ into one matrix $\Cb$, i.e. $\Cb = [\Ab, \Bb]$, then it holds that 
\begin{align*}
    \{\lambda_{1}(\Cb), \lambda_{2}(\Cb), \cdots, \lambda_{P_1+P_2}(\Cb)\} = \{\lambda_{1}(\Ab), \lambda_{2}(\Ab), \cdots, \lambda_{P_1}(\Ab)\} \cup \{\lambda_{1}(\Bb), \lambda_{2}(\Bb), \cdots, \lambda_{P_2}(\Bb)\},
\end{align*}
where $\lambda_i(\cdot)$ denote the $i$-th singular value of a matrix in descending order. 
\end{lemma}

\begin{proof}[Proof of Lemma~\ref{lemma:singular_value_concatenation}]
    By the connection between the definition of singular value and eigenvalue, we can make spectral decomposition for both $\Ab^\top\Ab$ and $\Bb^\top\Bb$ as,
    \begin{align*}
        \Ab^\top\Ab =\sum_{i=1}^{P_1} \lambda_i^2(\Ab) \vb_{A,i}\vb_{A,i}^\top; \quad \Bb^\top\Bb =\sum_{i=1}^{P_2} \lambda_i^2(\Bb) \vb_{B,i}\vb_{B,i}^\top,
    \end{align*}
    where $\vb_{A,i}$ is the $i$-th eigenvector corresponding $\lambda_i^2(\Ab)$, and $\vb_{B,i}$ is the $i$-th eigenvector corresponding $\lambda_i^2(\Bb)$. By the orthogonality between $\Ab$ and $\Bb$, we can derive that 
    \begin{align*}
        \Cb^\top \Cb = \begin{bmatrix}
\Ab^\top \Ab & \mathbf{0}_{P_1\times P_2} \\
\mathbf{0}_{P_2\times P_1} & \Bb^\top\Bb 
\end{bmatrix}.
    \end{align*}
    We further generate $\tilde{\vb_{A,i}} \in \RR^{P_1+P_2}$ for all $i\in [P_1]$ by concatenating zero vector to the end of $\vb_{A,i}$, i.e.
    $\tilde{\vb_{A,i}} = [\vb_{A,i}^\top, \mathbf{0}_{P_2}^\top]^\top$. Similarly, we also generate $\tilde{\vb_{B,i}} \in \RR^{P_1+P_2}$ for all $i\in [P_2]$ by concatenating zero vector at the start of $\vb_{B,i}$, i.e.
    $\tilde{\vb_{B,i}} = [\mathbf{0}_{P_1}^\top, \vb_{B,i}^\top]^\top$. Then we can rewrite $\Cb^\top \Cb$ as
    \begin{align*}
        \Cb^\top \Cb = \sum_{i=1}^{P_1} \lambda_i^2(\Ab) \tilde{\vb_{A,i}}\tilde{\vb_{A,i}}^\top + \sum_{i=1}^{P_2} \lambda_i^2(\Bb) \tilde{\vb_{B,i}}\tilde{\vb_{B,i}}^\top,
    \end{align*}
    which is the spectral decomposition of $\Cb^\top \Cb$ by the fact that all $\tilde{\vb_{A,i}}$'s and $\tilde{\vb_{B,i}}$'s are normalized and orthogonal to each other. Since the singular value of $\Cb$ is the square root of the eigenvalue of $\Cb^\top \Cb$, we complete the proof. 
\end{proof}

\begin{lemma}[Min-max principle of singular value]\label{lemma:min_max_principle}
Let $\Ab$ be a matrix and $\lambda_k(\Ab)$ be its $k$-th singular values in descending order. Then it holds that 
\begin{align*}
    &\lambda_k(\Ab) =\max_{S:\mathrm{dim}(S) = k}\min_{\xb\in S, \|\xb\|_2=1}\big\|\Ab\xb\big\|_2;\\
    &\lambda_k(\Ab) =\min_{S:\mathrm{dim}(S) = d-k+1}\max_{\xb\in S, \|\xb\|_2=1}\big\|\Ab\xb\big\|_2.
\end{align*}
\end{lemma}

\bibliography{ref}

\bibliographystyle{ims}
\end{document}